\documentclass[twoside]{article}

\usepackage[arxiv]{prep}
\usepackage[round]{natbib}

\usepackage{amsmath,amsfonts,bm}

\def\rh{{\textnormal{h}}}
\def\rk{{\textnormal{k}}}
\def\rr{{\textnormal{r}}}
\def\rs{{\textnormal{s}}}
\def\rx{{\textnormal{x}}}
\def\ry{{\textnormal{y}}}
\def\rz{{\textnormal{z}}}

\def\rvr{{\mathbf{r}}}
\def\rvx{{\mathbf{x}}}

\def\rvz{{\mathbf{z}}}

\def\vx{{\bm{x}}}
\def\vy{{\bm{y}}}

\def\sL{{\mathbb{L}}}
\def\sN{{\mathbb{N}}}
\def\sR{{\mathbb{R}}}
\def\sX{{\mathbb{X}}}
\def\sY{{\mathbb{Y}}}

\newcommand{\laplace}{\mathrm{Laplace}} %
\newcommand{\bernoulli}{\mathrm{Bern}} %
\newcommand{\normal}{\mathrm{Normal}}
\newcommand{\poisson}{\mathrm{Poisson}}
\newcommand{\invgauss}{\mathrm{IG}}
\newcommand{\negbin}{\mathrm{NB}}

\newcommand{\E}{\mathbb{E}}

\newcommand{\KL}{\mathrm{KL}}
\newcommand{\kl}{\mathrm{kl}}

\newcommand{\datadistro}{D}
\newcommand{\prior}{Q_0}
\newcommand{\priormarg}{Q_{\mathrm{marg}}}
\newcommand{\posterior}{ {Q_n} }
\newcommand{\posteriori}{ {Q_i} }
\newcommand{\dataspace}{\mathcal Z}
\newcommand{\hypospace}{\mathcal H}
\newcommand{\distrospace}{\mathcal M}
\newcommand{\distr}{P} 
\newcommand{\distra}[1]{\distr_{#1}}
\newcommand{\funcspace}{\mathcal F}
\newcommand{\bounddistrospace}{\mathcal P}
\newcommand{\convfuncs}{\mathcal C}

\newcommand{\bounddistrospacebern}{\bounddistrospace_{\textrm{Bern}}}
\newcommand{\bounddistrospacenorm}{\bounddistrospace_{\textrm{Norm},\sigma^2 } }
\newcommand{\bounddistrospacenormprime}{\bounddistrospace_{\textrm{Norm},{\sigma'}^2 } }
\newcommand{\bounddistrospacepoi}{\bounddistrospace_{\textrm{Poi}}}
\newcommand{\bounddistrospacegamma}{\bounddistrospace_{\Gamma}}
\newcommand{\bounddistrospacelaplace}{\bounddistrospace_{\textrm{Lap}}}
\newcommand{\bounddistrospaceinvgauss}{\bounddistrospace_{\textrm{IG}}}
\newcommand{\bounddistrospacenegbin}{\bounddistrospace_{\textrm{NB}}}

\newcommand{\trainset}{\rvz}

\newcommand{\trainlossarg}[2]{ R_{\,#1}(#2) }
\newcommand{\btrainlossarg}[2]{ \bar{R}_{#1}(#2) }
\newcommand{\htrainlossarg}[2]{ \widehat{R}_{#1}(#2) }
\newcommand{\trainlossh}{ \trainlossarg{\trainset}{h} }
\newcommand{\trainlossrh}{ \trainlossarg{\trainset}{\rh} }
\newcommand{\trainlossQ}{ \btrainlossarg{\trainset}{\posterior} }
\newcommand{\trainlossQavg}{ \htrainlossarg{\trainset}{\posterior} }
\newcommand{\trainlossQavgi}{ \htrainlossarg{\rz_i\!}{\posteriori} }

\newcommand{\poplossarg}[1]{ {R}_{\datadistro}(#1) }
\newcommand{\bpoplossarg}[1]{ \bar{R}_{\datadistro}(#1) }
\newcommand{\hpoplossarg}[1]{ \widehat{R}_{\datadistro}(#1) }
\newcommand{\poplossh}{ \poplossarg{h} }
\newcommand{\poplossrh}{ \poplossarg{\rh} }
\newcommand{\poplossQ}{ \bpoplossarg{\posterior} }
\newcommand{\poplossQavg}{ \hpoplossarg{\posterior} }

\newcommand{\DeltatPoi}{\Delta^{p}_t}
\newcommand{\Deltacram}{\Delta^\Psi_{\bounddistrospace}}
\newcommand{\Deltacramarg}[1]{\Delta^\Psi_{#1}}

\newcommand{\IDn}{\Upsilon_\Delta(n)}
\newcommand{\ID}{\Upsilon_\Delta}
\newcommand{\gIDtwoarg}[2]{\Upsilon^{#2}_{#1}}
\newcommand{\gIDarg}[1]{\gIDtwoarg{#1}{\bounddistrospace}}
\newcommand{\gID}{\gIDarg{\Delta}}
\newcommand{\gIDn}{\gID(n)}

\newcommand{\gIDcram}{\Upsilon^{\bounddistrospace}_{\Deltacram}}
\newcommand{\gIDcramshortarg}[1]{\bar \Upsilon(#1)}
\newcommand{\gIDcramshort}{\gIDcramshortarg{\bounddistrospace}}
\newcommand{\KLterm}{\KL( \posterior \Vert \prior )}
\newcommand{\KLtermavg}{\KL( \posterior \datadistro^n \Vert \prior \datadistro^n )}
\newcommand{\KLtermavgi}{\KL( \posteriori \datadistro \Vert \prior \datadistro )}

\newcommand{\mutualinfo}{\mathrm{I}}

\newcommand{\setE}{\mathcal E}
\newcommand{\dv}{\mathrm{d}}
\newcommand{\ffunc}{\Phi}
\newcommand{\flint}{f^{\textnormal{lin}}_t}
\newcommand{\flintg}{f^{\textnormal{lin}}_{t,g}}
\newcommand{\tspace}{\mathcal T}
\newcommand{\ufunc}{\Xi}
\newcommand{\Psifunca}[1]{\Psi_{#1}}
\newcommand{\Bavg}{\widehat B}

\newcommand{\abs}[1]{\left\lvert #1 \right \rvert}
\newcommand{\mean}[1]{\bar{#1}}

\newcommand{\ie}{\emph{i.e.}}
\newcommand{\eg}{\emph{e.g.}}

\newcommand{\codelink}{\href{https://github.com/fredrikhellstrom/comparing-comparators/}{[\textcolor{orange}{GitHub}]}}

\usepackage[inline]{enumitem}

\usepackage{xcolor}
\usepackage{comment}
\usepackage{amssymb,amsmath,amsthm}
\usepackage{thmtools}
\usepackage{thm-restate}
\usepackage{graphicx}
\usepackage{subcaption}
\usepackage{mdframed}
\newmdenv[
  topline=false,
  bottomline=false,
  rightline=false,
  skipabove=\topsep,
  skipbelow=\topsep
]{leftrule}

\usepackage[hidelinks]{hyperref}

\newcommand{\doi}[1]{\textrm{doi:} \href{https://doi.org/#1}{#1}}

\makeatletter
\renewcommand\@oddhead{\hspace{0.49\textwidth}\thepage}
\renewcommand\@evenhead{\hspace{0.49\textwidth}\thepage}
\makeatother

\usepackage[capitalize]{cleveref}

\newtheorem{theorem}{Theorem}
\newtheorem{lemma}[theorem]{Lemma}

\newtheorem{definition}[theorem]{Definition}

\crefname{theorem}{Theorem}{Theorems}
\crefname{lemma}{Lemma}{Lemmas}
\crefname{remark}{Remark}{Remarks}
\crefname{corollary}{Corollary}{Corollaries}
\crefname{definition}{Definition}{Definitions}
\crefname{proposition}{Proposition}{Propositions}

\crefname{equation}{}{}

\vfuzz \hfuzz
\begin{document}

\onecolumn
\preptitle{Comparing Comparators in Generalization Bounds}

\prepauthor{ Fredrik Hellstr\"om \And Benjamin Guedj } %

\prepaddress{ University College London \\ \textrm{f.hellstrom@ucl.ac.uk} \And Inria and University College London \\ \textrm{b.guedj@ucl.ac.uk} } 

\begin{abstract}
We derive generic information-theoretic and PAC-Bayesian generalization bounds involving an arbitrary convex \emph{comparator} function, which measures the discrepancy between the training and population loss.
The bounds hold under the assumption that the cumulant-generating function (CGF) of the comparator is upper-bounded by the corresponding CGF within a family of bounding distributions.
We show that the tightest possible bound is obtained with the comparator being the convex conjugate of the CGF of the bounding distribution, also known as the Cram\'er function.
This conclusion applies more broadly to generalization bounds with a similar structure.
This confirms the near-optimality of known bounds for bounded and sub-Gaussian losses and leads to novel bounds under other bounding distributions.
\end{abstract}

\section{Introduction}

A key question in statistical learning theory is that of \emph{generalization}: how can we certify that a hypothesis with good performance on training data has similarly good performance on new, unseen data?
More explicitly, when does a low training loss imply a low population loss?
A standard approach is to express the population loss as the sum of the training loss and the \emph{generalization gap}, \ie, the difference between population and training loss, and derive a bound on the generalization gap.
With this decomposition, the discrepancy between training and population loss is measured through their difference.
While simple and intuitive, this is often far from being the most effective approach---one may instead measure the discrepancy between the training and population loss through an alternative \emph{comparator} function, of which the difference is a single specific case.
In this paper, we examine the choice of this comparator in detail, and propose a systematic approach to selecting the optimal one.

To concretize this discussion,
we consider the Probably Approximately Correct (PAC)-Bayes framework, originating in the seminal works of~\citet{shawetaylor-97a,mcallester-98a}.
This framework yields bounds on the population loss, averaged over a stochastic learning algorithm, that hold with high probability over the draw of the training data.
A particularly appealing feature of PAC-Bayesian generalization bounds is that they depend on the specific learning algorithm, distribution, and data set under consideration.
This is closely related to \emph{information-theoretic} generalization bounds, where the main focus has been on bounds in expectation, in which the loss is averaged both with respect to the learning algorithm and training data \citep{zhang-06a,russo-16a,xu-17a}.
We refer to \citet{guedj-19a,alquier-21a} for recent surveys on PAC-Bayes, and to the monograph by \citet{hellstrom-23a} for a broader discussion on generalization and links with information theory.
Although some of our results apply more broadly, we focus on PAC-Bayesian and information-theoretic bounds for clarity.
While there exists a wide array of PAC-Bayesian bounds, the majority can be derived through a generic result that takes a convex function as parameter.
To make this precise, we first need to introduce some notation.

\paragraph{\,}

\textcolor{white}{.}
\\[-1.1cm]

Consider a distribution $\datadistro$ on the instance space $\dataspace$, and let the training set~$\trainset=(\rz_1,\dots,\rz_n)$ be drawn from the product distribution~$\datadistro^n$.
Let $\distrospace(\hypospace)$ denote the set of probability measures on the hypothesis space $\hypospace$.
The stochastic learning algorithm is represented through a distribution $\posterior\in\distrospace(\hypospace)$, called a \emph{posterior}.
Note that $\posterior$ is allowed to depend on~$\trainset$.\footnote{Formally, $\posterior$ is a Markov kernel with source $\dataspace^n$ and target $\hypospace$ (with associated $\sigma$-algebras).}
PAC-Bayesian bounds depend on a dissimilarity measure between $\posterior$ and a reference distribution $\prior\in\distrospace(\hypospace)$, called a \emph{prior}.
Typically, $\prior$ is independent from~$\trainset$, although this is not always the case.
While the terminology is inspired by the connection to Bayesian statistics, $\prior$ and $\posterior$ do not need to be related via Bayesian inference~\citep[see, \emph{e.g.},][for a discussion]{guedj-19a}.
However, we will require throughout that $\posterior$ is absolutely continuous with respect to $\prior$, denoted by $\posterior\ll\prior$.
The performance of a hypothesis is measured through a loss function $\ell: \hypospace\times\dataspace\rightarrow \sL \subseteq \sR^+$.
Without loss of generality, we will assume that~$\sL=[0,1]$ for bounded loss functions (arbitrary bounded loss functions can be recovered through affine transformations).
For a given hypothesis $h\in\hypospace$, the training loss $\trainlossh$ and population loss $\poplossh$ are given by
\begin{align}
    \trainlossh &= \frac1n \sum_{i=1}^n \ell(h,\rz_i),\label{eq:trainlossh-def}\\
    \poplossh &=  \E_{\rz\sim \datadistro} [ \ell(h,\rz) ] .\label{eq:poplossh-def}
\end{align}
The PAC-Bayesian training loss $\trainlossQ$ and population loss $\poplossQ$ are obtained as
\begin{align}
    \trainlossQ &= \E_{\rh \sim \posterior}[ \trainlossarg{\rvz}{\rh} ],\label{eq:trainlossQ-def}\\ 
    \poplossQ &=  \E_{\rh\sim \posterior} [ \poplossarg{\rh} ] .\label{eq:poplossQ-def}
\end{align}
We refer to convex functions $\Delta: \sL^2 \rightarrow \sR^+$ as \emph{comparator functions}.
Intuitively, 
a comparator function computes a discrepancy between the training and population loss.
With this notation in place, we are ready to state the generic PAC-Bayesian bound for bounded losses~\citep{germain-09a,begin-16a}.
\begin{restatable}{theorem}{boundedgenericpacbayes}\textnormal{(\citealp[Thm.~1]{begin-16a}).}\label{thm:bounded-generic-pac-bayes}
Consider a fixed prior~$\prior\in\distrospace(\hypospace)$, a convex comparator function $\Delta: \sL^2\rightarrow \sR^+$, 
and an uncertainty~$\delta\in(0,1)$.
Assume that~$\sL=[0,1]$.
Then, with probability~$1-\delta$ simultaneously for all~$\posterior$ such that $\posterior\ll\prior$, 
\begin{equation}\label{eq:generic-bound-bounded}
    \Delta\big( \trainlossQ , \poplossQ \big) \leq \frac{ \KLterm + \ln\frac{\IDn}{\delta}  }{n} 
\end{equation}
where $\KLterm$ is the KL divergence and
\begin{equation}\label{eq:IDn-def}
    \IDn = \sup_{r\in[0,1]} \sum_{k=0}^n \binom{n}{k}r^k (1-r)^{n-k} e^{n \Delta(k/n, r)}.
\end{equation}
If~$\trainlossQ=\alpha$,~$\KLterm\leq\beta$, and~$\IDn\leq \iota(n)$, this leads to the bound~$\poplossQ\leq  B^{\Delta}_{n}(\alpha,\beta,\iota)$, where
\begin{align}\label{eq:B-def}
    B^{\Delta}_{n}(\alpha,\beta,\iota) = \sup_{\rho\in\sL}\bigg\{ \rho : \Delta(\alpha, \rho) \leq \frac{\beta + \ln\frac{\iota(n)}{\delta} }{n} \bigg\}.
\end{align}
\end{restatable}
Here, the function $B^{\Delta}_n$ is essentially a numerical inversion of the bound in \cref{eq:generic-bound-bounded}: it outputs the largest possible value of the population loss that is consistent with the bound.
By suitably selecting the comparator function~$\Delta$ and controlling the resulting~$\ID$, several explicit bounds can be obtained.
The perhaps most intuitive choice is to simply consider the scaled difference, \ie, $\Delta_t(q,p)=t(p-q)$~\citep{mcallester-03a}.
However, other choices are likely to lead to tighter bounds.
For instance, with $\Delta(q,p)=C_\gamma(q,p)$ for $\gamma\in\sR$, where
\begin{equation}\label{eq:catoni-function}
    C_\gamma(q,p) = \gamma q - \ln(1-p+pe^\gamma),
\end{equation}
we find that, with probability $1-\delta$ for a fixed $\gamma$,
\begin{equation}\label{eq:catoni-bound}
    C_\gamma(\trainlossQ,\poplossQ) \leq \frac{ \KLterm + \ln\frac{1}{\delta}  }{n} .
\end{equation}
This is the family of \emph{Catoni bounds}~\citep{catoni-07a}.
Now, let $\bernoulli(p)$ denote a Bernoulli distribution with parameter $p$, and define the binary KL divergence as
\begin{align}
    \kl(q,p) &= \KL( \bernoulli(q) \Vert \bernoulli(p) ) \\
    &= q\ln \frac{q}{p} + (1-q)\ln\frac{1-q}{1-p} .
\end{align}
With $\Delta(q,p)=\kl(q,p)$, we obtain the MLS bound, named for~\citet{maurer-04a,langford-01a}:
\begin{equation}\label{eq:mls-bound}
    \kl(\trainlossQ,\poplossQ) \!\leq\! \frac{ \KLterm \!+\! \ln\frac{2\sqrt n}{\delta}  }{n} .\!
\end{equation}
This flexibility raises the question: which~$\Delta$ leads to the tightest bound on~$\poplossQ$?
Recently,~\citet{foong-21a} established the following: first, no choice of $\Delta$ in \cref{thm:bounded-generic-pac-bayes} can give a tighter bound than
\begin{align}\label{eq:catoni-bound-with-inf}
   \poplossQ \leq \inf_\gamma B^{C_\gamma}_{n}(\trainlossQ,\KLterm,1) .
\end{align}
Second, the right-hand side of \cref{eq:catoni-bound-with-inf} is
\begin{equation}\label{eq:inf-catoni-mls}
   \inf_\gamma B^{C_\gamma}_{n}(\trainlossQ,\KLterm,1) \\
   = B^{\kl}_{n}(\trainlossQ,\KLterm,1).
\end{equation}
This can be expressed as follows:
no bound on $\poplossQ$ based on \cref{thm:bounded-generic-pac-bayes} is tighter than the one obtained from \cref{eq:mls-bound} without the $\ln(2\sqrt n)/n$ term, and this bound is equivalent to~\cref{eq:catoni-bound} for the optimal $\gamma$.
Now, it is important to emphasize that the optimal bound in \cref{eq:catoni-bound-with-inf}, sometimes called the \emph{optimistic} MLS bound, has \emph{not} been proven to be valid: the optimal value of $\gamma$ in \cref{eq:catoni-bound} depends on the random variable $\trainlossQ$, and hence, taking the infimum in \cref{eq:catoni-bound-with-inf} is invalid without a union bound.
However, this demonstrates that, in this sense, \cref{eq:mls-bound} is optimal up to the logarithmic term.
Note that the assumption of the loss function being bounded is central to these results, and it is unclear what can be said in more general settings with unbounded losses.

\textbf{Overview and contributions.} 
Based on the preceding discussion, the following questions naturally arise:\\
\begin{enumerate*}[label=(\roman*)]
\item Why does optimistic MLS yield the tightest bound?
\item What is the optimal $\Delta$ beyond bounded losses? \\%
\end{enumerate*}
In this paper, we answer these questions as follows.
In \cref{sec:average}, we consider the average setting, enabling us to state our conclusions in a simpler form.
First, %
we derive a generic generalization bound in terms of any convex comparator for which the cumulant-generating function (CGF) is bounded by the corresponding CGF from a family of bounding distributions.
We prove that the optimal comparator is the \emph{Cram\'er function}---\ie, the conxvex conjugate of the CGF---of the bounding distribution.
If the bounding distributions form a natural exponential family (NEF), the Cram\'er function is a KL divergence.
In \cref{sec:pac-bayes}, we turn to the PAC-Bayesian setting. 
We derive an analogous generic generalization bound, %
and establish that the same Cram\'er function is near-optimal (up to a logarithmic term). %
As special cases, we recover the conclusions of \citet{foong-21a} for bounded losses and establish the optimality of the bound from \citet{xu-17a} for sub-Gaussian losses.
In \cref{sec:applications}, we specialize our approach to obtain generalization bounds for sub-Poissonian, sub-gamma, and sub-Laplacian losses, and in \cref{sec:numeric}, we numerically evaluate these bounds.
A summary of our notation, along with useful facts about information theory, convex analysis, and NEFs, is provided in \cref{app:useful-facts}.
The proofs of all of our results are deferred to \cref{app:proofs}. We close with additional theoretical and experimental results in \cref{app:more-results}.

\textbf{Related work.} 
PAC-Bayesian bounds for bounded losses with the difference-comparator were initially studied by \citet{shawetaylor-97a,mcallester-98a,mcallester-03a}.
Subsequently, \cite{langford-01a,maurer-04a,catoni-07a} considered alternative comparators for bounded losses, leading to \cref{eq:catoni-bound} and \cref{eq:mls-bound}.
\citet{zhang-06a} derived bounds for potentially unbounded losses using a comparator based on the CGF of the loss evaluated at $1$, %
and established a relation between average and PAC-Bayesian bounds via exponential inequalities (explored in-depth in \citealp{grunwald-23a}).
Bounds with generic comparators for bounded losses
were obtained by \citet{germain-09a,begin-16a}, and extended to unbounded losses by \citet{rivasplata-20a}.
General tail behaviors beyond bounded losses were also considered by, \eg, \citet{germain-16b,alquier-17a,bu-20a,mhammedi-20a,banerjee-21a,haddouche-21a,haddouche2022supermartingales,wu-23-arxiv,rodriguezgalvez-23a,lugosi-23a}.
However, the optimal comparator choice was not studied in any of these works.
Most closely related to this paper is \citet{foong-21a}, where comparator optimality was studied for bounded losses.

\section{Average Bounds and the Optimal Comparator}
\label{sec:average}

As aforementioned, we will first consider \emph{average} generalization bounds.
In this section, we thus consider the average training and population loss, given by
\begin{align}
    \trainlossQavg &= \E_{ \rh , \rvz \sim \posterior  \datadistro^n}[ \trainlossarg{\rvz}{\rh} ],\label{eq:trainlossQavg-def}\\ 
    \poplossQavg &=  \E_{ \rh , \rvz \sim \posterior  \datadistro^n } [ \poplossarg{\rh} ] .\label{eq:poplossQavg-def}
\end{align}
Here, $\posterior\datadistro^n$ is the product distribution on $\hypospace\times\dataspace^n$ induced by $\posterior$ and $\datadistro^n$.

\subsection{A Generic Average Generalization Bound}

\begin{restatable}{theorem}{genericaverage}\label{thm:cgf-generic-average}
Let $\bounddistrospace$ be a set of distributions such that, for all $r \in \sL$, there exists a $\distra{r}\in \bounddistrospace$ with first moment~$r$.
Let $\convfuncs$ denote the set of functions from $\sR^2$ to~$\sR$ that are proper, convex, and lower semicontinuous.\footnote{Functions defined on a subset of $\sR^2$ are extended by setting them to be $+\infty$ outside of the original domain.}
For any $\rvx=(\rx_1,\dots,\rx_n)$, let $\mean{\rvx} = \sum_{i=1}^n \rx_i/n$.
Furthermore, let $\funcspace\subseteq\convfuncs$ denote the subset of $\convfuncs$ such that, for all $h\in\hypospace$ and $f\in\funcspace$,
\begin{align}\label{eq:bounding-function-condition}
\!\!\!\E_{\,\rvz \sim \datadistro^n}[e^{f(\trainlossh,\poplossh)}]\!\leq\! \E_{\,\rvx \sim \distra{\poplossh}^n}[e^{f(\mean{\rvx},\poplossh)}].\!
\end{align}
Then, for all~$\Delta\in\funcspace$ and all~$\posterior$ such that $\posterior\ll\prior$, 
\begin{align}
 \Delta\big( \trainlossQavg , \poplossQavg ) \leq \frac{ \KL( \posterior \datadistro^n \Vert \prior  \datadistro^n ) \!+\! \ln \gIDn }{n} .\label{eq:thm-cgf-generic-average}
\end{align}
Here, $\gIDn = \sup_{r\in\sL} \E_{ \, \rvx \sim \distra{r}^n} \exp\!\left({n \Delta( \bar\rvx, r )}  \right)$.
\end{restatable}
If $\sL=[0,1]$, the condition in \cref{eq:bounding-function-condition} holds with $\bounddistrospace$ as
\begin{equation}\label{eq:P-Bernoulli}
    \bounddistrospacebern = \{ \bernoulli(r): r\in[0,1] \}
\end{equation}
and $\funcspace=\convfuncs$ \citep[Lemma~3]{maurer-04a}.
Thus, \cref{thm:cgf-generic-average} includes an average version of \cref{thm:bounded-generic-pac-bayes} as a special case.
Note that, if~$\prior$ is set to be the true marginal distribution $\priormarg$ induced on $\rh$ by $\posterior\datadistro^n$, \ie, for any measurable~$\setE \subset \hypospace$,
\begin{equation}
    \priormarg(\setE) = \int_{\dataspace^n} \posterior(\setE) \dv\datadistro^n(\rvz),
\end{equation}
we have that $ \KL( \posterior \datadistro^n \Vert \priormarg  \datadistro^n ) = \mutualinfo(\rh;\rvz)$ is the mutual information.
Hence, $ \KL( \posterior \datadistro^n \Vert \prior  \datadistro^n )$ can be seen as a mutual information with a mismatched marginal.
By the golden formula for mutual information, we have $\mutualinfo(\rh;\rvz)\leq  \KL( \posterior \datadistro^n \Vert \prior  \datadistro^n )$ for any prior~$\prior\ll\priormarg$ (see \cref{lemma:golden-formula} in \cref{app:useful-facts}).

\subsection{Beyond Bounded: Sub-$\bounddistrospace$ Losses}

To see the relevance of \cref{thm:cgf-generic-average} beyond the case of bounded losses, we need to be more concrete regarding the set $\funcspace$ of admissible functions and the set $\bounddistrospace$ of bounding distributions.
Recall that $\sigma$-sub-Gaussian random variables are characterized by having a CGF that is dominated by the CGF of some Gaussian distribution with variance $\sigma^2$, with similar notions for, \eg, sub-gamma and sub-exponential random variables~\citep[Chapter~2]{wainwright-19a}.
In \cref{def:sub-p}, we extend this to general bounding distributions.
\begin{definition}[\textbf{Sub-$\bounddistrospace$ Losses}]\label{def:sub-p}
    Let $\bounddistrospace$ be a set of distributions such that, for all $r \in \sL$, there exists a $\distra{r}\in \bounddistrospace$ with first moment~$r$.
    Furthermore, for all $r\in\sL$, let $\tspace_r\subseteq\sR$ and $\tspace=\{\tspace_r:r\in\sL\}$.
    Then, we say that the loss is sub-$(\bounddistrospace,\tspace)$ if, for all~$h\in\hypospace$ and $t\in\tspace_{\poplossh}$, we have
    \begin{equation}\label{eq:cgf-t-bound-assume}
        \E_{\,\rz\sim\datadistro}[\exp\!\big(t\ell(h,\rz)\big)] \leq \E_{\,\rx\sim\distra{\poplossh}}[\exp(t\rx)].
    \end{equation}
    If $\tspace_r=\sR$ for all $r\in\sL$, we say that the loss is sub-$\bounddistrospace$.
\end{definition}
Note that the condition in~\cref{eq:cgf-t-bound-assume} corresponds to assuming that the CGF of the loss is dominated by the CGF of the bounding distribution for all $t\in\tspace_{\poplossh}$.
In the language of Theorem 2, this corresponds to saying that the function $\flintg(q,p)=tq+g(p)$ is in $\mathcal F$ for all $t\in\tspace_{\poplossh}$ and all functions $g:\sL\rightarrow \sR$.
As indicated, sub-Gaussian random variables can be expressed as sub-$\bounddistrospacenorm$, where $\bounddistrospacenorm$ is the set of Gaussian distributions with a fixed variance $\sigma^2\in\sR^+$:
\begin{equation}\label{eq:nef-gaussian}
    \bounddistrospacenorm = \{ \normal(\mu, \sigma^2): \mu\in\sR \}.
\end{equation}
For a given loss, there are often multiple valid choices of $\mathcal P$.
Bounded losses, for instance, are both sub-Bernoulli and, by Hoeffding's lemma, sub-Gaussian.
Furthermore, for any $\sigma'>\sigma$, sub-$\bounddistrospacenorm$ losses are also sub-$\bounddistrospacenormprime$.
However, selecting $\mathcal P$ to be the family that most tightly bounds the true CGF will naturally yield the tightest bound.

Unlike for the case of a bounded loss, assuming a bound on the CGF does not in general guarantee that $\funcspace$ contains all of $\convfuncs$.
However, it does imply that $\funcspace$ contains a wide array of functions, including all totally monotone functions and all infinitely differentiable functions whose derivatives of all orders are non-negative. 
To the best of our knowledge, this includes all comparator functions that have been considered in the literature.
We provide a more detailed characterization in \cref{propo:cgf-implies-funcspace} in \cref{app:additional-theoretical}.
In any case, assuming that $\flintg\in\funcspace$ is sufficient to find the optimal comparator function in \cref{thm:cgf-generic-average}, as we show next.

\subsection{The Optimal Comparator Function}

Recall that the convex conjugate of a function~$f$ is
\begin{equation}
    f^*(\vy) = \sup_{\vx} \big\{ \langle\vx, \vy\rangle - f(\vx) \big\} ,
\end{equation}
where $\langle\cdot,\cdot\rangle$ is the inner product.
For $f\in\convfuncs$, $(f^*)^*=f$.
\begin{restatable}{theorem}{bestcomparatoraverage}\label{thm:best-comparator-cgf-average}
Assume that the loss is sub-$(\bounddistrospace,\tspace)$.
Let $\Psifunca{p}(t)=\ln \E_{\rx \sim \distra{p}}[e^{t \rx}]$ denote the CGF of the distribution~$\distra{p}$, and let $\Deltacram(q,p)$ be the Cram\'er function, \ie, the convex conjugate of $\Psifunca{p}$:
\begin{equation}\label{eq:deltacram-def}
\Deltacram(q,p) = \Psifunca{p}^*(q) = \sup_{t\in\tspace_p} \big\{ tq -\Psifunca{p}(t) \big\}.
\end{equation}
Furthermore, define
\begin{align}\label{eq:Bhat-def}
    \Bavg^{\Delta}_{n}(\alpha,\beta,\iota) = \sup_{\rho\in\sL}\bigg\{ \rho : \Delta(\alpha, \rho) \leq \frac{\beta + \ln\iota(n) }{n} \bigg\}.
\end{align}
Then, for any $\Delta\in\funcspace$, we have
\begin{align}\label{eq:avg-optimal-upper-bound}
 \!\!\! \poplossQavg &\!\leq\! \Bavg^{\Deltacram}_{n} \big(\trainlossQavg,\KLtermavg,1\big) \\\label{eq:avg-optimal-lower-bound}
 &\!\leq\! \Bavg^{\Delta}_{n} \big(\trainlossQavg,\KLtermavg,\gID\big) .
\end{align}
\end{restatable}
Note that $\Bavg^{\Delta}_{n}$ in \cref{eq:Bhat-def} is simply the average counter-part to $B^{\Delta}_{n}$ in \cref{eq:B-def}, and hence, without the $\delta$ term.
The result in \cref{thm:best-comparator-cgf-average} allows us to conclude that, using the generic bound in \cref{thm:cgf-generic-average}, the optimal average generalization bound is obtained by setting the comparator function to be the Cram\'er function.
Specifically, this is obtained by numerically inverting
\begin{equation}
\Deltacram(\trainlossQavg,\poplossQavg) \leq \frac{\KLtermavg}{n} 
\end{equation}
as described in \cref{eq:Bhat-def}.
For independent and identically distributed random variables, the Cram\'er function characterizes the probability of rare events~(\citealp{cramer-44a}, \citealp[Sec.~2.2]{boucheron-13a}).
Thus, the connection to generalization bounds is somewhat natural.

While we focus on information-theoretic and PAC-Bayesian bounds for concreteness, the conclusions of \cref{thm:best-comparator-cgf-average} hold more broadly for generalization bounds with a similar structure.
Specifically, if \cref{eq:thm-cgf-generic-average} holds with the KL divergence replaced by some other complexity measure, the same reasoning still applies.
For the case of Bernoulli distributions in \cref{eq:P-Bernoulli}, we have
\begin{equation}
    \Deltacramarg{\bounddistrospacebern}(q,p) = \kl(q,p) ,
\end{equation}
as can be shown via a straight-forward calculation.
As it turns out, a similar statement holds more generally as long as $\bounddistrospace$ is a natural exponential family (NEF).
A NEF is a set of probability distributions whose probability density (or mass) functions can be written
\begin{equation}\label{eq:nef-definition}
    p(x\vert \theta) = h(x) e^{\theta x - g(\theta) } ,
\end{equation}
where $h(x)$ and $g(\theta)$ are known functions and $\theta$ is the natural parameter.
A NEF can equivalently be described by its expectation parameter $\mu=g'(\theta)$, which equals its first moment (\citealp{nielsen-09a}, \citealp[Sec.~9.13.3]{wasserman-10a}).
Unless otherwise specified, we characterize NEFs using expectation parameters.
In the case where $\bounddistrospace$ is a NEF, Kullback's inequality becomes an equality \citep{kullback-54a}.
\begin{restatable}{proposition}{kullbackequality}\label{propo:kullback-equality}
Assume that $\bounddistrospace$ is a NEF.
Then,
\begin{equation}
\Deltacram(q,p) = \Psifunca{p}^*(q) = \KL(\distra{q} \Vert \distra{p} ) . 
\end{equation}
\end{restatable}
Thus, the optimal comparator function for bounded losses is the binary KL divergence (as in \citealp[Thm.~9]{hellstrom-22a}).
As another example, consider the set of Gaussian distributions with known variance in~\cref{eq:nef-gaussian}.
Then, the optimal comparator function is
\begin{equation}
   \! \KL\big(\normal(q,\sigma^2) \Vert \normal(p,\sigma^2) \big) = \frac{(q-p)^2}{2\sigma^2} ,
\end{equation}
as $\bounddistrospacenorm$ is a NEF.
This demonstrates the optimality of the bound in \citet[Thm.~1]{xu-17a}.
As discussed by \citet{foong-21a}, these optimal comparators are not necessarily unique.
We discuss further applications of the generic bound in \cref{sec:applications}.

\subsection{A Samplewise Generalization Bound}

By an altered derivation, one can obtain a bound in terms of a \emph{samplewise} KL divergence, akin to \citet{negrea-19a, bu-20a, haghifam-20a}.
While it is possible to obtain bounds in terms of arbitrary random subsets of $\trainset$, we focus on the samplewise case as it yields the tightest bound~\citep{rodriguezgalvez-20a,harutyunyan-21a}.
\begin{restatable}{theorem}{genericaveragedisintegrated}\label{thm:cgf-generic-average-disintegrated}
Consider the setting of \cref{thm:cgf-generic-average}.
Let~$\rvz_{-i}$ denote $\rvz$ with the $i$th element removed.
Let $\posteriori$ denote the distribution induced on $\rh$ when marginalizing over $\rvz_{-i}$, \ie, for any measurable~$\setE \subset \hypospace$,
\begin{equation}
    \posteriori(\setE) = \int_{\dataspace^{n-1}} \posterior(\setE) \dv\datadistro^{n-1}(\rvz_{-i}),
\end{equation}
Then, for all~$\Delta\in\funcspace$ and~$\posterior$ such that $\posteriori\ll\prior$, 
\begin{equation}\label{eq:thm-cgf-generic-average-disintegrated-B}
   \!\!\!\! \poplossQavg \!\leq\! \frac1n\! \sum_{i=1}^n  \Bavg^{\Delta}_{1}\!\!\left(\!\trainlossQavgi , \KLtermavgi , \!\gID \!\right) \!.\!\!\!
\end{equation}
\end{restatable}
Now, setting the prior to be the true marginal gives the samplewise mutual information $ \KL( \posteriori \datadistro \Vert \priormarg  \datadistro ) = \mutualinfo(\rh;\rz_i)$.
With this prior and $\gID=1$, the bound in \cref{thm:cgf-generic-average-disintegrated} is always at least as tight as the one in \cref{thm:cgf-generic-average}, as we show in Appendix~\ref{app:additional-theoretical}.

\section{Generic PAC-Bayesian Bound for Sub-$\bounddistrospace$ Losses}
\label{sec:pac-bayes}

Having introduced the main ideas in the average setting, we now turn to PAC-Bayesian bounds.
We follow a strategy similar to the one presented in \cref{sec:average}, but the additional difficulty of handling the randomness of the training data calls for a more elaborate treatment.

\subsection{A Generic PAC-Bayesian Bound}

We begin by deriving a version of \cref{thm:bounded-generic-pac-bayes} that holds under the assumption that the CGF of the comparator under the true data distribution is bounded by the CGF under a certain bounding distribution---\ie, a PAC-Bayesian variant of \cref{thm:cgf-generic-average}.
\begin{restatable}{theorem}{genericpacbayes}\label{thm:cgf-generic-pac-bayes}
Let $\bounddistrospace$, $\funcspace$ and $\gID$ be as in \cref{thm:cgf-generic-average}.
Consider a fixed function $\Delta\in\funcspace$.
Then, with probability~$1-\delta$ simultaneously for all~$\posterior$ such that $\posterior\ll\prior$, 
\begin{equation}\label{eq:thm-cgf-generic-pac-bayes}
   \!\! \Delta( \trainlossQ , \poplossQ ) \leq \frac{ \KL( \posterior \Vert \prior ) + \ln\frac{\gIDn}\delta }{n} .\!\!
\end{equation}
\end{restatable}
The bound in~\cref{eq:thm-cgf-generic-pac-bayes} is similar to the generic PAC-Bayesian bound from \citet{rivasplata-20a}, but with a more explicit bound on the CGF term therein.
For $\sL=[0,1]$ and $\bounddistrospace$ being the Bernoulli distributions, \cref{thm:bounded-generic-pac-bayes} is recovered as a special case.

\subsection{The Near-Optimal Comparator}

We are now ready to present a characterization of the near-optimal bound obtainable via \cref{thm:cgf-generic-pac-bayes}.
Specifically, in \cref{eq:cgf-generic-pac-bayes-best-lower}, we state a lower limit on the bound that can be obtained from \cref{thm:cgf-generic-pac-bayes} in terms of the Cram\'er function.
Then, in \cref{eq:cgf-generic-pac-bayes-best-upper-parametric}, we derive a parametric bound, which is used to obtain explicit bounds in terms of the Cram\'er function in \cref{eq:cgf-generic-pac-bayes-best-upper,eq:cgf-generic-pac-bayes-best-ufunc-def,eq:cgf-generic-pac-bayes-best-upper-un}.

\begin{restatable}{theorem}{bestcomparator}\label{thm:best-comparator-cgf-pac-bayes}
Assume that the loss is sub-$(\bounddistrospace,\tspace)$.
Then, for any $\Delta\in\funcspace$ in \cref{thm:cgf-generic-pac-bayes},
\begin{equation}
B^{\Deltacram}_{n}(\trainlossQ,\KLterm,1) \leq B^{\Delta}_{n}(\trainlossQ,\KLterm,\gID) . \label{eq:cgf-generic-pac-bayes-best-lower}
\end{equation}
Furthermore, with $\gIDcramshort := \gIDcram$, we have
\begin{equation}\label{eq:cgf-generic-pac-bayes-best-upper-chernoff}
    \poplossQ \leq B^{\Deltacram}_{n} \big(\trainlossQ,\KLterm,\gIDcramshort\big) .
\end{equation}
Finally, for all $t\in\tspace_p$, let $\Delta^t_{\bounddistrospace}(q,p)=tq - \Psifunca{p}(t)$.
Then, for any fixed $t$, we have
\begin{equation}\label{eq:cgf-generic-pac-bayes-best-upper-parametric}
    \poplossQ \leq B^{\Delta^t_{\bounddistrospace}}_{n} \big(\trainlossQ,\KLterm,1\big) .
\end{equation}
\end{restatable}
Here,~\cref{eq:cgf-generic-pac-bayes-best-lower} demonstrates that no choice of $\Delta$ leads to a tighter bound than what is obtained with $\Deltacram$, provided that $\gIDcramshort$ is replaced by $1$.
This is analogous to \citet[Thm.~4]{foong-21a}, with the crucial difference that~\cref{eq:cgf-generic-pac-bayes-best-lower} holds beyond bounded losses.
While \cref{eq:cgf-generic-pac-bayes-best-lower} is not shown to be a valid generalization bound,~\cref{eq:cgf-generic-pac-bayes-best-upper-chernoff} provides a valid bound in terms of $\gIDcramshort$.
Hence, the result in \cref{thm:best-comparator-cgf-pac-bayes} demonstrates that, potentially up to the $\gIDcramshort$-dependent term, the optimal bound on $\poplossQ$ obtainable from \cref{thm:cgf-generic-pac-bayes} is obtained by setting the comparator to be the Cram\'er function.
For the special case of bounded losses, \cref{eq:cgf-generic-pac-bayes-best-upper-chernoff} reduces to the MLS bound in \cref{eq:mls-bound}, while \cref{eq:cgf-generic-pac-bayes-best-upper-parametric} reduces to the Catoni bound in \cref{eq:catoni-bound}.

Next, we use \cref{eq:cgf-generic-pac-bayes-best-upper-parametric} to obtain upper bounds in terms of $\Deltacram$, but with explicit expressions in place of~$\gIDcramshort$.
\begin{restatable}{corollary}{unionbasedbounds}\label{cor:union-based-bounds}
    Assume that $\KLterm \leq u(n)$ or that $n\trainlossQ \leq u(n)$ for a function $u:\sN\rightarrow \sR^+$.
    Then, we have
\begin{equation}\label{eq:cgf-generic-pac-bayes-best-upper-un}
   \!\! \poplossQ \leq B^{\Deltacram}_{n} \big(\trainlossQ,\KLterm, 2e\lceil u\rceil \big) .
\end{equation}
For any value of $\KLterm$ and $\trainlossQ$, we have
\begin{equation}\label{eq:cgf-generic-pac-bayes-best-upper}
    \poplossQ \leq B^{\Deltacram}_{n} \big(\trainlossQ,\KLterm, \ufunc \big) 
\end{equation}
where
\begin{equation}\label{eq:cgf-generic-pac-bayes-best-ufunc-def}
    \ufunc = \frac{\pi^2  (1+\min\{n\trainlossQ, \KLterm\})^2}{3} .
\end{equation}
\end{restatable}
The bound in \cref{eq:cgf-generic-pac-bayes-best-upper-un} is essentially a variation of \cref{eq:cgf-generic-pac-bayes-best-upper-chernoff}.
To shed light on this comparison, consider the bounded loss case.
Specifically, if $\sL=[0,1]$ and $\Delta(q,p)\leq 1$ for all $q,p\in\sL$---as is the case for \cref{eq:mls-bound}---it is sufficient to consider~$u(n)=n$, since the boundedness of the loss implies $n\trainlossQ\leq n$.
Thus, we recover \cref{eq:mls-bound} but with $\ln(2en)/n$ in place of $\ln(\sqrt{2n})/n$.
As argued by \citet{rodriguezgalvez-23a}, $u(n)=n$ is also a reasonable choice for more general settings, as we are mainly interested in cases where $\KLterm/n\rightarrow 0$ as $n\rightarrow\infty$;~otherwise, our bound will not vanish as the number of training data increases. 
Note that the more benign dependence on $n$ in \cref{eq:mls-bound} stems from bounding $\gIDcramshort$ directly in \cref{eq:cgf-generic-pac-bayes-best-upper-chernoff} instead of starting from \cref{eq:cgf-generic-pac-bayes-best-upper-parametric}, with a similar situation for the sub-Gaussian case (cf.\ \citealp[corrected Cor.~2]{hellstrom-20b} and \citealp[Thm.~10]{rodriguezgalvez-23a}).
The upside of \cref{eq:cgf-generic-pac-bayes-best-upper-un} is that it leads to explicit bounds without necessitating a bound on~$\gIDcramshort$.
The bound in \cref{eq:cgf-generic-pac-bayes-best-upper} can potentially be tighter than \cref{eq:cgf-generic-pac-bayes-best-upper-un} if either $\trainlossQ$ or $\KLterm$ are small.
The appearance of the KL term is similar to \citet[Thm.~6]{seldin-12a}, who obtained a similar dependence in a PAC-Bayes bound based on Azuma's inequality, while the bound with the training loss in \cref{eq:cgf-generic-pac-bayes-best-upper} is, to the best of our knowledge, new.
For the bounded loss setting, if the minimum in \cref{eq:cgf-generic-pac-bayes-best-ufunc-def} is 0, we recover \cref{eq:mls-bound} but with $\ln(\pi^2/3)/n$ instead of $\ln(\sqrt{2n})/n$, leading to an improved bound for $n>5$.

\section{Applications}
\label{sec:applications}

So far, we have used the comparator characterization in \cref{thm:best-comparator-cgf-average} and \cref{thm:best-comparator-cgf-pac-bayes} to shed light on the bounded loss case and verify the (near-)optimality of known bounds for sub-Gaussian losses.
We now apply our general techniques to other bounding distributions, and present new explicit generalization bounds.
Specifically, we consider sub-Poissonian, sub-gamma, and sub-Laplacian losses.
As Poisson and gamma distributions are both NEFs, the relevant Cram\'er functions can be expressed as KL divergences, as per \cref{propo:kullback-equality}.
Since Laplace distributions with different first moments do not form a NEF, the relevant Cram\'er function is not a KL divergence for this case.
The average bounds that we present are optimal in the sense of \cref{thm:best-comparator-cgf-average}, while the PAC-Bayesian bounds are near-optimal in the sense of \cref{thm:best-comparator-cgf-pac-bayes}.
In \cref{app:additional-theoretical}, we also present explicit bounds for sub-inverse Gaussian and sub-negative binomial losses.

\subsection{Sub-Poissonian Losses}

We begin by considering losses that are sub-$\bounddistrospacepoi$, with $\bounddistrospacepoi$ being the set of Poisson distributions:
\begin{equation}\label{eq:poisson-distro}
    \bounddistrospacepoi = \{  \poisson(\mu): \mu \in \sR^+ \} .
\end{equation}
With this, we obtain the following.
\begin{restatable}{corollary}{poissonbound}\label{cor:poisson-bounds}
Assume that the loss is sub-$\bounddistrospacepoi$, as defined in~\cref{eq:poisson-distro}.
Define $\Deltacramarg{\bounddistrospacepoi}$ as
\begin{align}\label{eq:poisson-kl}
\Deltacramarg{\bounddistrospacepoi}(q,p) &= \KL\big(\poisson(q) \Vert \poisson(p)  \big) \\
&= p-q+q\ln\frac{q}{p}.
\end{align}
Then, we have the average bound
\begin{equation}\label{eq:subpoisson-average}
    \Deltacramarg{\bounddistrospacepoi}( \trainlossQavg , \poplossQavg )
    \leq \frac{ \KL( \posterior \datadistro^n \Vert \prior  \datadistro^n ) }{n} .
\end{equation}
Furthermore, with probability $1-\delta$, we have the PAC-Bayesian bound, with $\ufunc$ as defined in \cref{eq:cgf-generic-pac-bayes-best-ufunc-def},
\begin{equation}\label{eq:subpoisson-pac-bayes}
    \Deltacramarg{\bounddistrospacepoi}(\trainlossQ,\poplossQ) \leq \frac{\KLterm + \ln \frac{\Xi}{\delta} }{n}.
\end{equation}
\end{restatable}
For sub-Poissonian losses, \cref{eq:cgf-generic-pac-bayes-best-upper-chernoff} does not yield a finite bound, since $\gIDcramshortarg{\bounddistrospacepoi}$ is unbounded.
This demonstrates the usefulness of \cref{cor:union-based-bounds}, as it allows for finite tail bounds in terms of the near-optimal comparator, despite this unboundedness.

The bounds in terms of $\Deltacramarg{\bounddistrospacepoi}$ admit a closed-form solution.
Specifically, we have that
\begin{align}
    \Bavg^{\Deltacramarg{\bounddistrospacepoi}}_{n}(\alpha,\beta,1) = \alpha W\bigg(e^{1-\frac{\beta}{n\alpha}}\bigg) ,
\end{align}
where $W(\cdot)$ denotes the Lambert $W$ function.

One can also derive a bound based on the comparator $\DeltatPoi(q,p)=(1-e^{-t})q - tp$, which is chosen to ensure that the CGF is independent of the mean.
We present the resulting bound in the following corollary.
\begin{restatable}{corollary}{subpoissondiffbound}\label{cor:subpoisson-diffbased-bound}
Assume that the loss is sub-$\bounddistrospacepoi$, as defined in~\cref{eq:poisson-distro}.
Then, we have the average bound
\begin{equation}\label{eq:subpoisson-average-diffbased}
\!\!\!\poplossQavg \!\leq \inf_{t>0}\! \left\{ \frac{t\trainlossQavg}{1\!-\!e^{-t}} \!+\! \frac{\KLtermavg}{(1\!-\!e^{-t})n} \right\}\!.\!
\end{equation}
\end{restatable}

\subsection{Sub-Gamma Losses}

We now turn to sub-gamma losses with fixed shape parameter $k$, which can be viewed as being sub-$(\bounddistrospacegamma,\tspace^\Gamma)$ with $\tspace^\Gamma_\mu = [0,k/\mu)$ and
\begin{equation}\label{eq:sub-gamma-def}
    \bounddistrospacegamma = \big\{ \Gamma(k, \mu/k): \mu\in\sR \big\}.
\end{equation}
Since the mean of a gamma distribution is the product of its parameters, $\mu$ above is indeed the mean.
Note that sub-gamma random variables are often defined in a slightly different way, stated in terms of an upper bound on the CGF of the gamma distribution (cf.\ \citealp[Sec.~2.4]{boucheron-13a}).

Several average information-theoretic generalization and PAC-Bayesian bounds for sub-gamma losses have been considered in the literature, but they are all based on the scaled difference between the training and population loss, \ie, $\Delta_t(q,p) = t(p-q)$ \citep{germain-16b,banerjee-21a,wu-23-arxiv}.
A consequence of this is that, in order to evaluate the bounds, one needs to know \emph{both} parameters of the bounding gamma distribution, which implies that one also has a bound on the mean.
Indeed, the supremum over $r\in\sL$ in the definition of $\gID$ precludes the use of $\Delta_t$, as $\gIDarg{\Delta_t}$ is unbounded.
Here, we instead consider
\begin{align}\label{eq:kl-gamma}
\Deltacramarg{\Gamma}(q,p) &= \KL\bigg( \Gamma\big(k,q/k\big) \, \big\Vert \, \Gamma\big(k,p/k \big) \bigg) \\
&= k\left( \frac{q}{p} - 1 - \ln \frac{q}{p} \right).
\end{align}
With this, we obtain the following bounds, which only depend on the shape factor $k$ in \cref{eq:sub-gamma-def}.

\begin{restatable}{corollary}{subgammabound}\label{cor:subgamma-bounds}
Assume that the loss is sub-$(\bounddistrospacegamma,\tspace^\Gamma)$.
Then, we have the average bound
\begin{equation}\label{eq:subgamma-average}
    \Deltacramarg{\Gamma}( \trainlossQavg , \poplossQavg )\
    \leq \frac{ \KL( \posterior \datadistro^n \Vert \prior  \datadistro^n ) }{n} .
\end{equation}
Furthermore, with probability $1-\delta$, we have the PAC-Bayesian bound
\begin{equation}\label{eq:subgamma-pac-bayes}
    \Deltacramarg{\Gamma}(\trainlossQ,\poplossQ) \leq \frac{\KLterm + \ln \frac{\Xi}{\delta} }{n}.
\end{equation}
\end{restatable}
To the best of our knowledge, \cref{cor:subgamma-bounds} provides the first PAC-Bayesian and information-theoretic generalization bounds for sub-gamma losses that do not require knowledge of both parameters of the bounding distribution.
Note that, since $\gIDcramshortarg{\bounddistrospacegamma}$ is unbounded,
\cref{eq:cgf-generic-pac-bayes-best-upper-chernoff} does not yield a finite bound.

\subsection{Sub-Laplacian Losses}

As a final example, we consider losses that are sub-$(\bounddistrospacelaplace,\tspace^b)$, where $\tspace^b_\mu=[0, 1/b)$ for all $\mu\in\sR$ and
\begin{equation}
    \bounddistrospacelaplace = \{ \laplace(\mu, b): \mu \in \sR \}
\end{equation}
are the Laplace distributions with mean $\mu$ and fixed scale parameter $b$.
Note that the Laplace distributions form an exponential family only if $\mu$ is fixed, and hence, $\bounddistrospacelaplace$ is not a NEF.
Therefore, the optimal comparator is not a KL divergence, but the Cram\'er function can still be computed as
\begin{equation}\label{eq:cramer-laplace}
    \Deltacramarg{\textrm{Lap}}(q,p) = \frac{\sqrt {(q-p)^2+b^2}}{b} - 1 
    + \ln\!\bigg( \frac{2(b\sqrt{(q-p)^2+b^2} -b^2)}{(q-p)^2} \bigg).
\end{equation}
With this, we obtain the following.
\begin{restatable}{corollary}{sublaplacebound}\label{cor:sublaplace-bounds}
Assume that the loss is sub-$(\bounddistrospacelaplace,\tspace^b)$.
Then, we have the average bound
\begin{equation}\label{eq:sublaplace-average}
    \Deltacramarg{\textrm{Lap}}( \trainlossQavg , \poplossQavg )\
    \leq \frac{ \KL( \posterior \datadistro^n \Vert \prior  \datadistro^n ) }{n} .
\end{equation}
\end{restatable}
While similar PAC-Bayesian results hold, we state only the average result for brevity.

For sub-Laplacian losses, the comparator $\Delta_t(q,p) = t(p-q)$ can also be used, as we show in the following.
\begin{restatable}{corollary}{sublaplacediffbound}\label{cor:sublaplace-diffbased-bound}
Assume that the loss is sub-$(\bounddistrospacelaplace,\tspace^b)$.
Then, we have the average bound
\begin{equation}
    \poplossQavg - \trainlossQavg 
  \!  \leq \!\!\inf_{t\in \! \big(0,\frac{1}{b}\big)} \!\!\bigg\{ \frac{\KLtermavg}{nt} - \frac{\ln(1-b^2t^2)}{t} \bigg\} .\label{eq:sublaplace-average-diffbased}
\end{equation}
\end{restatable}
As it turns out, the bound based on \cref{eq:sublaplace-average-diffbased} is \emph{identical} to the one based on \cref{eq:sublaplace-average}.
The reason for this is the particular form of the CGF of the Laplace distribution, as established in the following proposition.
\begin{restatable}{proposition}{exponentialcgf}\label{propo:exponential-cgf}
Assume that the CGF for any distribution $P_r\in \bounddistrospace$ and $t\in\tspace$ can be written as
\begin{equation}\label{eq:condition-exponential-cgf}
   \ln \E_{\rx\sim P_r}[ e^{t\rx} ] = tr + \ln g(t^2),
\end{equation}
where $g(t^2)$ does not depend on the mean $r$.
Then,
\begin{equation}
    \Bavg^{\Deltacram}_{n} \big(\alpha,\beta,1\big) 
 = \inf_t \Bavg^{\Delta_t}_{n} \big(\alpha,\beta,\gIDarg{\Delta_t}\big) .
\end{equation}
\end{restatable}
Since $\bounddistrospacelaplace$ satisfies \cref{eq:condition-exponential-cgf}, the claimed equivalence of the bounds based on \cref{eq:sublaplace-average} and \cref{eq:sublaplace-average-diffbased} follows.

\section{Numerical Evaluation}
\label{sec:numeric}

As established, the bounds based on the Cram\'er function are near-optimal in the sense of \cref{thm:best-comparator-cgf-average,thm:best-comparator-cgf-pac-bayes}.
Still, it is interesting to evaluate their quantitative advantage compared to, \eg, bounds based on the (scaled) difference-comparator.
In this section, we evaluate this discrepancy numerically.
For simplicity, we focus on average bounds, but similar conclusions apply for the PAC-Bayesian case.

It is well-established in the literature that PAC-Bayesian and information-theoretic bounds can give accurate loss estimates and be used to construct learning algorithms for many settings, including neural networks~\citep{langford-01b,ambroladze-06a,dziugaite-17a,neyshabur-18a,letarte-19a,zhou-18a,biggs-21a,biggs-22c,dziugaite-21a,harutyunyan-21a,perez-ortiz-21a,lotfi-22a,biggs-22a,wang-23a}.
Thus, instead of studying any specific setting, we evaluate the bounds while varying the relevant inputs: the training loss $\trainlossQavg$ and the normalized KL divergence, \ie, $\KLtermavg/n$.
This provides a wider perspective, as any specific setting can be identified with a subset of these input values.

\begin{figure}
\centering
\begin{subfigure}{0.43\textwidth}
\centering
    \includegraphics[width=0.99\textwidth]{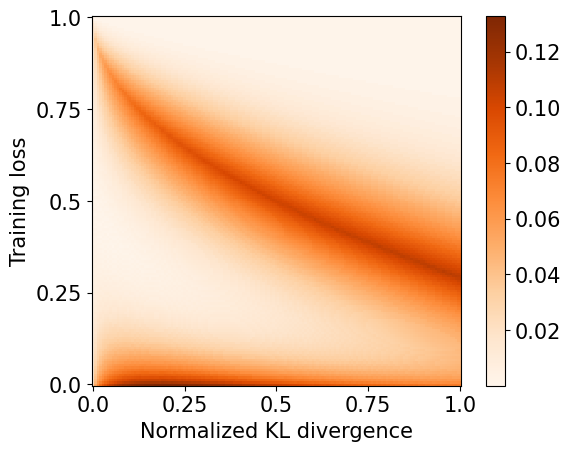}
    \caption{The difference between the binary KL bound on the population loss and the one obtained via sub-Gaussianity for a bounded loss function.}
    \label{fig:subbernoulli}
\end{subfigure}%
\hspace{0.1\textwidth}
\begin{subfigure}{0.43\textwidth}
\centering
    \includegraphics[width=0.99\textwidth]{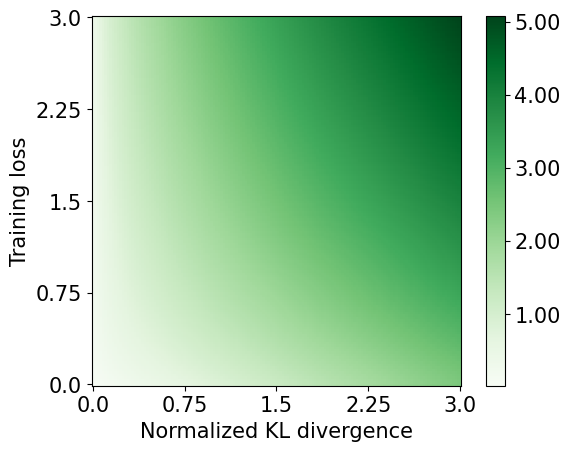}
    \caption{The difference between the sub-Poissonian bound on the population loss obtained via \cref{eq:subpoisson-average-diffbased} and the one obtained via \cref{eq:subpoisson-average}.}
    \label{fig:subpoisson}
\end{subfigure}
\caption{Numerical evaluations for \cref{sec:numeric}.}
\end{figure}

To begin, we consider sub-Bernoulli losses---that is, bounded losses.
As mentioned, the Cram\'er function in this case is the binary KL divergence $\kl(q,p)$, while the difference comparator $\Delta_t(q,p)$ leads to the sub-Gaussian bound (since bounded losses are $1/2$-sub-Gaussian).
To compare these bounds, we evaluate
\begin{equation}\label{eq:bounded-comparison-plot}
\min\big\{1, \big( \alpha + \sqrt{\beta/2n} \big) \big\} - \Bavg^\kl_n(\alpha,\beta,1),
\end{equation}
where $\alpha$ is the training loss and $\beta/n$ is the normalized KL divergence.
This is illustrated in \cref{fig:subbernoulli}.
When both $\alpha$ and $\beta$ are high, both bounds lead to the trivial upper bound of $1$, and are thus equal.
The binary KL bound is most clearly advantageous for small training losses and in the region where the sub-Gaussian bound becomes trivial.

In \cref{fig:subpoisson}, we consider sub-Poissonian losses, and numerically evaluate the discrepancy between the bound based on \cref{eq:subpoisson-average} and the one based on \cref{eq:subpoisson-average-diffbased}, that is,
\begin{equation}\label{eq:subpoisson-discrepancy}
\inf_t \big\{ \Bavg^{\Delta_t}_{n}(\alpha,\beta, \gIDarg{\Delta_t} ) \big\} - \Bavg^{\Deltacramarg{\bounddistrospacepoi}}_{n}(\alpha,\beta,1) .
\end{equation}
Since \cref{eq:subpoisson-average} is optimal, \cref{eq:subpoisson-discrepancy} is non-negative for all values.
The biggest discrepancy arises when the training loss and the normalized KL divergence are both high.
If one removes the minimum in \cref{eq:bounded-comparison-plot}, the same behavior emerges for bounded losses (see \cref{app:additional-numerical}).

For sub-gamma losses, it is unclear how to construct an alternative comparator.
Indeed, for any comparator based on a scaled difference of population and training loss, the CGF depends on the true mean, and is unbounded when taking the supremum.
One approach, taken by \citet{germain-16b}, is to assume that both parameters of the bounding distribution, and hence its mean, are bounded.
However, this necessitates stronger assumptions on the true loss distribution.
In light of this discussion, we simply present the values of the bound based on \cref{eq:subgamma-average} in \cref{app:additional-numerical}, where we also evaluate bounds for other bounding distributions and study the $n$-dependence of the bounds based on the Cram\'er function.
Code for reproducing all of our figures is available on \codelink.

\section{Discussion and Outlook}

In this paper, we studied the optimal comparator function for generalization bounds under CGF constraints.
For PAC-Bayesian bounds, we showed that the bounds in terms of the Cram\'er function are near-optimal up to a logarithmic term.
In a subsequent paper, \citet{casado-24a} showed that this term is always at most logarithmic in $n$, provided that $n$ in the denominator in \cref{eq:thm-cgf-generic-pac-bayes} is replaced by $n-1$.
Whether or not this dependence can be improved to a constant remains an open question which, as discussed by \citet{foong-21a}, is relevant for the small-data regime.
Furthermore, the use of almost exchangeable priors, which gives rise to average bounds in terms of the conditional mutual information, has proven fruitful to obtain tighter bounds for bounded losses~\citep{audibert-04a,catoni-07a,steinke-20a,haghifam-22a}, with some work on improving comparators~\citep{hellstrom-22a}.
Combining this with our techniques may shed further light on the comparator choice.
Finally, while we considered generalization bounds under CGF constraints, this precludes heavy-tailed losses.
Extending our analysis to generalization bounds under moment constraints, for instance, is a promising avenue for future studies.

\subsubsection*{Acknowledgements}
\noindent F.H.\ acknowledges support by the Wallenberg AI, Autonomous Systems and Software Program (WASP) funded by the Knut and Alice Wallenberg Foundation.
B.G.\ acknowledges partial support by the U.S. Army Research Laboratory and the U.S. Army Research Office, and by the U.K. Ministry of Defence and the U.K. Engineering and Physical Sciences Research Council (EPSRC) under grant number EP/R013616/1. B.G. acknowledges partial support from the French National Agency for Research, through grants ANR-18-CE40-0016-01 and ANR-18-CE23-0015-02, and through the programme ``France 2030'' and PEPR IA on grant SHARP ANR-23-PEIA-0008. 

\bibliography{biblio}

\newpage
 \appendix

\section{Useful Facts}
\label{app:useful-facts}

In this section, we summarize the main notation used in the paper and provide some relevant background.

\subsection{Summary of Notation}\label{app:notation}

\renewcommand{\arraystretch}{1.3}
\begin{table}[h]
    \centering
    \begin{tabular}{l l r}
        Notation \;\;\;\;& Definition & First use  \\ \hline 
        $\posterior$ and $\prior$ & Posterior and prior & page 1 \\
        $\sL\subseteq \sR^+$ & Loss range & page 2\\
        $\trainlossh$ & $ \frac1n \sum_{i=1}^n \ell(h,\rz_i)$  & \cref{eq:trainlossh-def} \\
        $\poplossh$ & $\E_{ \rz \sim \datadistro } [ \ell(h, \rz) ]$  & \cref{eq:poplossh-def} \\
        $\trainlossQ$ & $\E_{\rh \sim \posterior}[ \trainlossarg{\rvz}{\rh} ]$  & \cref{eq:trainlossQ-def} \\
        $\poplossQ$ & $\E_{\rh\sim \posterior} [ \poplossarg{\rh} ]$  & \cref{eq:poplossQ-def} \\
        $\IDn$ & $\sup_{r\in[0,1]} \sum_{k=0}^n \binom{n}{k}r^k (1-r)^{n-k} e^{n \Delta(k/n, r)}$ &\cref{eq:IDn-def} \\
        $B^{\Delta}_{n}(\alpha,\beta,\iota)$ \;\;& $\sup_{\rho\in\sL}\big\{ \rho : \Delta(\alpha, \rho) \leq \frac{\beta + \ln\frac{\iota(n)}{\delta} }{n} \big\}$ & \cref{eq:B-def}\\
        $\trainlossQavg$ & $\E_{ \rh , \rvz \sim \posterior  \datadistro^n}[ \trainlossarg{\rz}{\rh} ]$ & \cref{eq:trainlossQavg-def}\\ 
    $\poplossQavg$ & $\E_{ \rh , \rvz \sim \posterior  \datadistro^n } [ \poplossarg{\rh} ]$ & \cref{eq:poplossQavg-def} \\
     $\gIDn $ & $ \sup_{r\in\sL} \E_{ \, \rvx \sim \distra{r}^n} \exp\!\left({n \Delta( \bar\rvx, r )}  \right)$ & Thm. \ref{thm:cgf-generic-average} \\
    $\Psifunca{p}(t)$ & $\ln \E_{X\sim \distra{p}}[e^{tX}]$ & Thm. \ref{thm:best-comparator-cgf-average} \\
    $\Deltacram(q,p)$ & $\Psifunca{p}^*(q)=\sup_{t\in\tspace_p} \big\{ tq -\Psifunca{p}(t) \big\}$ & \cref{eq:deltacram-def} \\
    $\Bavg^{\Delta}_{n}(\alpha,\beta,\iota)$ & $\sup_{\rho\in\sL}\big\{ \rho : \Delta(\alpha, \rho) \leq \frac{\beta + \ln\iota(n) }{n} \big\}$ & \cref{eq:Bhat-def}\\
    $\gIDcramshort$ & $\gIDcram$ & Thm. \ref{thm:best-comparator-cgf-pac-bayes} \\
    $\ufunc$ & $\pi^2  (1+\min\{\trainlossQ, \KLterm\})^2/3$ & \cref{eq:cgf-generic-pac-bayes-best-ufunc-def} \\
    $\ffunc_q(r)$ & $-\Psifunca{r}(q)$ & \cref{eq:phi-def} \\
    $A$ & $\sup_{ c_q,c_p\in \sR} \big\{  - \Delta^*(c_q, c_p)  + \ffunc^*_{c_q}(c_p)  \big\}$ & \cref{eq:A-def}
    \end{tabular}
    \caption{Summary of notation.}
    \label{tab:notation}
\end{table}
For reference, in \cref{tab:notation}, we summarize the main notation used throughout the main paper and the appendix.

\subsection{Information Theory}

We begin by providing some definitions and results from information theory.
More details are available, for instance, in \citet{cover-06a}.
First, we provide the definition of the KL divergence.%

\begin{definition}[KL divergence]\label{def:kl-divergence}
Let $P$ and $Q$ be two distributions such that $P\ll Q$.
Then, the KL divergence between $P$ and $Q$ is, with $\ln \frac{\dv P}{\dv Q}$ denoting the Radon-Nikodym derivative,
\begin{equation}
    \KL(P\Vert Q) = \int \dv P \ln \frac{\dv P}{\dv Q}.
\end{equation}
\end{definition}

\paragraph{\,}

\textcolor{white}{.}\\[-1.1cm]

Note that the KL divergence is non-negative, \ie, $\KL(P\Vert Q)\geq 0$.

If $\rx$ and $\ry$ are random variables with joint distribution $P_{\rx\ry}$ and product of marginals $P_\rx P_\ry$, $\KL(P_{\rx\ry} \Vert P_\rx P_\ry)= \mutualinfo(\rx;\ry)$ is the mutual information between $\rx$ and $\ry$.
The chain rule of mutual information states that, with a third random variable $\rz$, $\mutualinfo(\rx;\ry,\rz)=\mutualinfo(\rx;\rz) + \mutualinfo(\rx;\ry\vert \rz)$, where $\mutualinfo(\rx;\ry\vert \rz)$ is the conditional mutual information.
If $\rz$ is independent of either $\rx$ or $\ry$, we have $\mutualinfo(\rx;\ry) \leq \mutualinfo(\rx;\ry\vert \rz)$.

A cornerstone of information-theoretic and PAC-Bayesian analysis is the Donsker-Varadhan variational representation of the KL divergence~\citep{donsker-75a}.

\begin{lemma}[Donsker-Varadhan variational representation]\label{lemma:donsker}
Let~$Q$ be a probability distribution on a measurable space~$\sX$, and let~$\Pi$ denote the set of probability measures such that, for all~$P\in\Pi$, we have~$P\ll Q$.
For every measurable function~$f:\sX\rightarrow \sR$ such that~$\E_{\rx\sim Q}[e^{ f(\rx)}]<\infty$, we have
\begin{equation}
 \ln \E_{\rx\sim Q}[e^{ f(\rx)}] =  \sup_{P\in \Pi} \bigg\{\E_{\rx\sim P}[f(\rx)] - \KL(P\Vert Q) \bigg\}.
\end{equation}
The supremum is attained by the \emph{Gibbs distribution}~$G$, which for any measurable $\setE\subset \sX$ is given by
\begin{equation}
\dv G(\setE) = \frac{\int_\setE e^{f(x)} \dv Q(x) }{\E_{\rx\sim Q}[e^{f(\rx)}]}.
\end{equation}
\end{lemma}

Finally, we present the golden formula for mutual information~\citep[Eq.~8.7]{csiszar-11a}.

\begin{lemma}[Golden formula for mutual information]\label{lemma:golden-formula}
Consider two random variables $\rx$ on $\sX$ and $\ry$ on $\sY$ with joint distribution $P_{\rx\ry}$ and marginal distributions $P_\rx$ and $P_\ry$.
Let $Q_\rx$ be a distribution on $\sX$, such that $\rx\sim Q_\rx$ is independent of $\ry$.
Then,
\begin{equation}
    \mutualinfo(\rx; \ry) = \KL( P_{\rx\ry} \Vert P_\rx P_\ry) \leq \KL( P_{\rx\ry} \Vert Q_\rx P_\ry) .
\end{equation}
\end{lemma}

\subsection{Convex Analysis}

The convex conjugate of a function $f:\sX\rightarrow \sY$ is defined as
\begin{equation}
    f^*(\vy) = \sup_{\vx\in\sX} \big\{ \langle\vx, \vy\rangle - f(\vx) \big\} ,
\end{equation}
where $\langle\cdot,\cdot\rangle$ is the inner product.
The convex conjugate of any function is convex and lower semicontinuous.
Recall that $\convfuncs$ denotes the set of functions that are convex, proper, and lower semicontinuous.
For $f\in\convfuncs$, $(f^*)^*=f$.
The convex conjugate is order-reversing in the following sense:
if, for two functions $f$ and $g$, we have $f(\vx)\leq g(\vx)$ for all $\vx\in\sX$, we have $f^*(\vy)\geq g^*(\vy)$ for all $\vy\in\sY$. 
For a more comprehensive overview, see \citet{rockafellar-70a}.

\subsection{Natural Exponential Families}

An natural exponential family (NEF) is a set of probability distributions whose probability density (or mass) functions can be written
\begin{equation}
    p(x\vert \theta) = h(x) e^{\theta x - g(\theta) } ,
\end{equation}
where $h(x)$ and $g(\theta)$ are known functions and $\theta$ is the natural parameter.
The function $g(\theta)$ is referred to as the log-normalizer.
The CGF for a distribution $P$ in a NEF is given by
\begin{equation}
    \Psi_{P}(t) = \ln \E_{\rx\sim P} [e^{t\rx}] = g(\theta+t) - g(\theta).
\end{equation}
This implies that the mean can be computed as $g'(\theta)$.
Further details are available in \citet{nielsen-09a} and \citet[Sec.~9.13.3]{wasserman-10a}.

\section{Proofs}
\label{app:proofs}

In this section, we provide the proofs of all results from the main paper.
For convenience, we repeat the statement of each result prior to proving it, demarcated by a horizontal rule on the left side. 
Note that the equation numbering in these repetitions coincides with the numbering used in the main paper, to avoid any possible confusion.

\subsection{Proofs for \cref{sec:average}}
\begin{leftrule}
\genericaverage*
\end{leftrule}
\begin{proof}
As $\Delta$ is convex, Jensen's inequality implies that
\begin{equation}\label{eq:genericaveragepf1}
    \Delta( \trainlossQavg , \poplossQavg ) \leq \E_{ \rh , \rvz \sim \posterior  \datadistro^n} \big[ \Delta( \trainlossrh, \poplossrh ) \big].
\end{equation}
Next, we use the Donsker-Varadhan variational representation of the KL divergence (\cref{lemma:donsker}) to obtain
\begin{equation}
\E_{ \rh , \rvz \sim \posterior  \datadistro^n} \big[ \Delta( \trainlossrh, \poplossrh ) \big] \leq \frac{\KLtermavg +  \ln \E_{ \rh , \rvz \sim \prior  \datadistro^n} \big[ e^{n \Delta( \trainlossrh, \poplossrh )} \big] }{n} .
\end{equation}
Next, we replace the expectation over the prior by the supremum:
\begin{equation}\label{eq:genericaveragepfmiddle}
\ln \E_{ \rh , \rvz \sim \prior  \datadistro^n} \big[ e^{n \Delta( \trainlossrh, \poplossrh )} \big] \leq \sup_{h\in\hypospace} \ln \big[ \E_{ \rvz \sim \datadistro^n} e^{n \Delta( \trainlossh, \poplossh )} \big] .
\end{equation}
By the assumption that $\Delta\in\funcspace$, we have
\begin{equation}
\sup_{h\in\hypospace} \ln \E_{ \rvz \sim \datadistro^n} \big[ e^{n \Delta( \trainlossh, \poplossh )} \big] \leq \sup_{h\in\hypospace} \ln \E_{ \rvr'_h \sim \distra{\poplossh}^n} \big[ e^{n \Delta( \mean{\rvr}'_h, \poplossh )} \big] .
\end{equation}
Finally, as the highest population loss is no greater than the highest loss, we get
\begin{equation}\label{eq:genericaveragepflast}
\sup_{h\in\hypospace} \ln \E_{ \rvr'_h \sim \distra{\poplossh}^n} \big[ e^{n \Delta( \mean{\rvr}'_h, \poplossh )}\big] \leq \sup_{r\in\sL}\ln \E_{ \rvx \sim \distra{r}^n} \big[ e^{n \Delta( \mean{\rvx}, r )} \big] .
\end{equation}
The result follows by combining \cref{eq:genericaveragepf1} to \cref{eq:genericaveragepflast}.
\end{proof}

\begin{leftrule}
\bestcomparatoraverage*
\end{leftrule}
\begin{proof}

We begin by proving the upper bound in \cref{eq:avg-optimal-upper-bound}.
First, we use the fact that the moment-generating function for a sum of independent random variables factorizes, so that
\begin{equation}
    \E_{\rvx\sim \distra{p}^n}[ e^{ nt\mean{\rvx} } ] = \left(\E_{\rx\sim \distra{p}}[ e^{ t\rx } ]\right)^n .
\end{equation}
By definition, for any fixed $t$, we have
\begin{equation}\label{eq:re-organize-cgf-average}
   \E_{\rx \sim \distra{p}}[ e^{ t \rx - \Psifunca{p}(t)  } ] = 1 .
\end{equation}
Hence, by applying \cref{thm:cgf-generic-average} with $\Delta(q,p)=tq-\Psifunca{p}(t)$, we find that with $\Delta(q,p) = tq- \Psifunca{p}(t)$ for any fixed $t\in\tspace$, we have
\begin{equation}
   t\trainlossQavg - \Psifunca{\poplossQavg}(t) \leq \frac{ \KLtermavg}{n} .
\end{equation}
Since this holds for any $t\in\tspace$, it also holds for the supremum.
Hence,
\begin{equation}
   \Deltacram(\trainlossQavg,\poplossQavg) = \sup_{t\in\tspace} \big\{ t\trainlossQavg - \Psifunca{\poplossQavg}(t) \big\} \leq \frac{ \KLtermavg}{n} .
\end{equation}
This establishes \cref{eq:avg-optimal-upper-bound}.

We now turn to proving the lower bound in \cref{eq:avg-optimal-lower-bound}.
To do this, we will show that, for any choice of $\Delta$ in \cref{thm:cgf-generic-average}, the resulting bound on $\poplossQavg$ is no better than the stated lower bound.
The proof consists of three steps:
\begin{enumerate*}[label=(\roman*)]
\item lower-bounding $\gID$, 
\item upper-bounding $\Delta$, and
\item putting every thing together. 
\end{enumerate*}
This roughly follows along the same lines as the proof of \citet[Thm.~4]{foong-21a}, with key modifications and subtleties that arise due to considering unbounded loss functions.
For convenience, we introduce
\begin{equation}\label{eq:phi-def}
    \ffunc_q(r) = -\Psifunca{r}(q).
\end{equation}

\textit{(i): Lower-bounding $\gID$.}\\
Since $\Delta$ is in $\convfuncs$, we have
\begin{equation}
\Delta(q,p) = \Delta^{**}(q,p) = \sup_{c_q,c_p\in \sR} (c_q q + c_p p - \Delta^*(c_q, c_p)).
\end{equation}
Recall that $\mean{\rvx} = \frac1n\sum_{i=1}^n \rx_i$.
Then, we have
\begin{align}
    \gIDn &= \sup_{r\in\sL} \E_{\rvx \sim \distra{r}^n} e^{n\Delta(\mean{\rvx},r)} \\
    &= \sup_{r\in\sL} \E_{\rvx \sim \distra{r}^n} e^{ \sup_{c_q,c_p\in \sR} (c_q \sum_i\rx_i + nc_p r - n\Delta^*(c_q, c_p)) } \\
     &\geq \sup_{r\in\sL, c_q,c_p\in \sR}  e^{ nc_p r - n\Delta^*(c_q, c_p) }\E_{\rvx \sim \distra{r}^n} e^{ c_q \sum_i\rx_i} \\
     &= \sup_{r\in\sL, c_q,c_p\in \sR}  e^{ nc_p r - n\Delta^*(c_q, c_p) }\exp\left(n\Psifunca{r}(c_q)\right)\\
     &= \sup_{r\in\sL, c_q,c_p\in \sR}  e^{ nc_p r - n\Delta^*(c_q, c_p) }\exp\left(-n\ffunc_{c_q}(r)\right).
\end{align}
Hence, we obtain
\begin{align}
\frac{\ln \gIDn}{n} &\geq \sup_{ c_q,c_p\in \sR} \big\{  - \Delta^*(c_q, c_p)  +  \sup_{r\in\sL}\big[c_p r - \ffunc_{c_q}(r) \big] \big\} \\
&= \sup_{ c_q,c_p\in \sR} \big\{  - \Delta^*(c_q, c_p)  + \ffunc^*_{c_q}(c_p)  \big\} \\
&:=  A.\label{eq:A-def}
\end{align}
Note that, since $\Delta$ is proper, $A$ is finite.

\textit{(ii): Upper-bounding $\Delta$.}\\
Define $\tilde \Delta^*$ as
\begin{equation}
    \tilde \Delta^*(c_p,c_q) = -  A + \ffunc^*_{c_q}(c_p).
\end{equation}
Since $A$ is finite, $\tilde \Delta^*$ is proper.
Furthermore, as it is an affine transformation of a convex conjugate, it is convex and lower semicontinuous.
Hence, $\tilde \Delta^*\in\convfuncs$, which implies that $(\tilde \Delta^*)^{**}=\tilde \Delta^*$.
This motivates the notation $\tilde \Delta=(\tilde \Delta^*)^*$.
With this, we obtain
\begin{align}
    \tilde \Delta(q,p) &=  A + \sup_{c_p,c_q\in\sR} [c_q q + c_p p - \ffunc^*_{c_q}(c_p)] \\
    &=  A + \sup_{c_q\in\sR} [c_q q + \ffunc_{c_q}(p)  ] \\
    &=  A + \sup_{c_q\in\sR} [ c_q q - \Psifunca{p}(c_q) ] \\
    &=  A + \Psifunca{p}^*(q)  .
\end{align}
We now need to show that $\tilde \Delta(q,p)\geq \Delta(q,p)$ for all $q,p\in\sL$:
\begin{align}
    -\tilde\Delta^*(c_q,c_p) + \ffunc^*_{c_q}(c_p) &=  A \\
&=\sup_{ c_q,c_p\in \sR} \big\{  - \Delta^*(c_q, c_p) + \ffunc^*_{c_q}(c_p) \big\} \\
&\geq  - \Delta^*(c_q, c_p)  + \ffunc^*_{c_q}(c_p) .
\end{align}
Therefore, we have $\tilde \Delta^*\leq \Delta^*$, which implies $\tilde \Delta \geq \Delta$ by the order-reversing property of the convex conjugate.

\textit{(iii): Putting everything together.} \\ 
First, since $\tilde \Delta \geq \Delta$, we have
\begin{equation}\label{eq:pf-bavg-start}
   \Bavg^{\tilde \Delta}_{n}(\trainlossQ,\KLterm,\gID) \leq  \Bavg^{\Delta}_{n}(\trainlossQ,\KLterm,\gID) .
\end{equation}
Furthermore, since $\ln\big(\gIDn\big)/n\geq A$, we have
\begin{equation}
   \Bavg^{\tilde \Delta}_{n}(\trainlossQ,\KLterm,e^{nA}) \leq  \Bavg^{\tilde \Delta}_{n}(\trainlossQ,\KLterm,\gID ).
\end{equation}
Finally, since $\tilde \Delta(q,p) = A + \Psifunca{p}^*(q) = A + \Deltacram(q,p)$, we obtain
\begin{equation}\label{eq:pf-bavg-end}
   \Bavg^{\tilde \Delta}_{n}(\trainlossQ,\KLterm,e^{nA}) = \Bavg^{\Deltacram}_{n}(\trainlossQ,\KLterm,1) ,
\end{equation}
leading to the final lower bound.

\end{proof}

\begin{leftrule}
\kullbackequality*
\end{leftrule}

\begin{proof}

For completeness, we begin by proving Kullback's inequality~\citep{kullback-54a}.
Let $P$ and $Q$ be two distributions such that $P\ll Q$.
Let $Q_\alpha$ be defined so that, for every measurable set $\setE$,
\begin{equation}\label{eq:theta-transform}
    Q_\alpha(\setE) = \frac{ \int_\setE e^{\alpha x} Q(dx) }{\int_\sR e^{\alpha x} Q(dx)}  = \frac{1}{M_Q(\alpha)} \int_\setE e^{\alpha x} Q(dx),
\end{equation}
where $M_Q(\alpha)$ denotes the moment-generating function of $Q$.
Then, we find that
\begin{align}
    \KL(P\Vert Q) &= \int_\sR \dv P \ln \left( 
\frac{\dv P}{\dv Q} \right) \\
 &= \int_\sR \dv P \ln \left( 
\frac{\dv P}{\dv Q} \frac{\dv Q_\alpha}{\dv Q_\alpha} \right) \\
    &= \KL(P\Vert Q_\alpha) + \int_\sR \dv P \ln \left( \frac{\dv Q_\alpha}{\dv Q } \right) .
\end{align}
The last term can be decomposed as
\begin{align}
\int_\sR \dv P \ln \left( \frac{\dv Q_\alpha}{\dv Q } \right) &= \int_\sR \dv P \ln \left( \frac{ e^{\alpha x} }{ M_Q(\alpha) } \right) \\
&= \alpha\mu_P - \Psi_Q(\alpha),
\end{align}
where $\mu_P$ denotes the first moment of $P$ and $\Psi_Q(\alpha)$ is the CGF of $Q$.
Now, due to the non-negativity of the KL divergence, we have
\begin{align}
    \KL(P\Vert Q) &=  \KL(P\Vert Q_\alpha) + \alpha\mu_P - \Psi_Q(\alpha)\\
    &\geq \alpha\mu_P - \Psi_Q(\alpha) .
\end{align}
Finally, by taking the supremum over $\alpha$, we obtain Kullback's inequality:
\begin{align}\label{eq:kullback-in-proof}
    \KL(P\Vert Q) \geq \sup_\alpha \big\{ \alpha\mu_P - \Psi_Q(\alpha) \big\} = \Psi^*_Q(\mu_P).
\end{align}
To establish the desired result, we need to show that the above is an equality provided that the distributions are in the same NEF.
Thus, assume that $P$ and $Q$ are in a NEF, with natural parameters $\theta_P$ and $\theta_Q$ respectively.
Denote the first moment of $P$ as $p$ and the first moment of of $Q$ as $q$.

First, observe that since $Q$ is in a NEF with parameter $\theta_Q$, as defined in \cref{eq:nef-definition}, the transformation in \cref{eq:theta-transform} gives another member of the NEF, but with parameter $\theta_Q+\alpha$.
Since $Q$ is in a NEF, the CGF is
\begin{equation}
    \Psi_Q(t) = g(\theta_Q+t) - g(\theta_Q).
\end{equation}
In particular, the first moment is $q=g'(\theta_Q)$.
Hence, the transformation in \cref{eq:theta-transform} leads to a distribution with first moment $g'(\theta_Q+\alpha)$.
If we set $\alpha=\theta_P-\theta_Q$, we thus obtain a distribution with first moment $p=g'(\theta_P)$---and since it is in the same NEF, $Q_{\theta_P-\theta_Q} = P$.
From this, it follows that $\KL(P\Vert Q_{\theta_P-\theta_Q})=0$.
Therefore, by following the same procedure as above,
\begin{align}
    \KL(P\Vert Q) &=  \KL(P\Vert Q_{\theta_P-\theta_Q}) + (\theta_P-\theta_Q)\mu_P - \Psi_Q(\theta_P-\theta_Q)\\
    &= (\theta_P-\theta_Q)\mu_P - \Psi_Q(\theta_P-\theta_Q) .
\end{align}
Now, since the general upper bound of Kullback's inequality in \cref{eq:kullback-in-proof} holds, and equality is achieved with $\alpha=(\theta_P-\theta_Q)$, it follows that this must be the supremum over $\alpha$.
Thus, we conclude
\begin{align}
    \KL(P\Vert Q) &= \sup_\alpha \big\{ \alpha \mu_P - \Psi_Q(\alpha) \big\} \\
    &= \Psi_Q^*(\mu_P).
\end{align}
\end{proof}

\begin{leftrule}
\genericaveragedisintegrated*
\end{leftrule}
\begin{proof}
As in the proof of \cref{thm:cgf-generic-average}, we begin by using the convexity of $\Delta$ and Jensen's inequality to conclude that
\begin{equation}\label{eq:genericaveragepfdisint1}
    \Delta( \trainlossQavg , \poplossQavg ) \leq \E_{ \rh , \rvz \sim \posterior  \datadistro^n} \big[ \Delta( \trainlossrh, \poplossrh ) \big].
\end{equation}
Now, recall that $\trainlossrh = \frac1n\sum_{i=1}^n \ell(\rh,\rz_i)$.
Thus, by using Jensen's inequality again,
\begin{equation}
\E_{ \rh , \rvz \sim \posterior  \datadistro^n} \big[ \Delta( \trainlossrh, \poplossrh )\big] \leq \E_{ \rh , \rvz \sim \posterior  \datadistro^n} \bigg[\frac1n \sum_{i=1}^n \Delta( \ell(\rh,\rz_i), \poplossrh ) \bigg].
\end{equation}
By the linearity of expectation and marginalizing,
\begin{equation}
\E_{ \rh , \rvz \sim \posterior  \datadistro^n} \bigg[ \frac1n \sum_{i=1}^n \Delta( \ell(\rh,\rz_i), \poplossrh )\bigg] = \frac1n \sum_{i=1}^n \E_{ \rh , \rz_i \sim \posteriori  \datadistro^n}  \big[ \Delta( \ell(\rh,\rz_i), \poplossrh ) \big].
\end{equation}
The proof now essentially proceeds as in \cref{thm:cgf-generic-average}, but for each term in the sum.
First, by using the Donsker-Varadhan variational representation of the KL divergence (\cref{lemma:donsker}), we obtain
\begin{equation}
\E_{ \rh , \rz_i \sim \posterior  \datadistro} \big[ \Delta( \ell(\rh,\rz_i), \poplossrh ) \big] \leq \KLtermavgi +  \ln \E_{ \rh , \rz_i \sim \prior  \datadistro} \big[ e^{\Delta( \ell(\rh,\rz_i), \poplossrh )} \big] .
\end{equation}
By replacing the expectation over the prior by the supremum, using the assumption that $\Delta\in\funcspace$, and the fact that the highest population loss is no greater than the highest loss, we get
\begin{equation}\label{eq:genericaveragepdfisintlast}
    \ln \E_{ \rh , \rz_i \sim \prior  \datadistro} \big[ e^{\Delta( \ell(\rh,\rz_i), \poplossrh )} \big] \leq \sup_{r\in\sL} \ln \E_{ \rx \sim  P_r} \big[ e^{\Delta( \rx, r )} \big] .
\end{equation}
The result follows by combining \cref{eq:genericaveragepfdisint1} to \cref{eq:genericaveragepdfisintlast}.
\end{proof}

\subsection{Proofs for \cref{sec:pac-bayes}}
 
\begin{leftrule}
 \genericpacbayes*
\end{leftrule}
 \begin{proof}
The proof essentially follows the same lines as \cref{thm:cgf-generic-average}, with an additional application of Markov's inequality.
Recall that $\mean{\rvx} = \frac1n\sum_{i=1}^n \rx_i$.
Then, by Jensen's inequality and the Donsker-Varadhan variational representation (\cref{lemma:donsker}),
\begin{align}
\Delta( \trainlossQ , \poplossQ ) &\leq \E_{\rh \sim \posterior} \big[ \Delta( \trainlossrh, \poplossrh ) \big] \\
    &\leq \frac{\KL(\posterior\Vert\prior) +  \ln \E_{\rh \sim \prior} \big[ e^{n \Delta( \trainlossrh, \poplossrh )} \big] }{n} .
\end{align}
Now, by Markov's inequality, we have that with probability $1-\delta$,
\begin{equation}
    \ln \E_{\rh \sim \prior} \big[ e^{n \Delta( \trainlossrh, \poplossrh )} \big] \leq \ln \bigg(\frac1\delta\E_{\rh \sim \prior,  \rvz \sim \datadistro^n} \big[ e^{n \Delta( \trainlossrh, \poplossrh )} \big] \bigg).
\end{equation}
The remaining steps are identical to \cref{eq:genericaveragepfmiddle} to \cref{eq:genericaveragepflast}, after which the result follows.
\end{proof}

\begin{leftrule}
\bestcomparator*
\end{leftrule}

\begin{proof}
We begin with \cref{eq:cgf-generic-pac-bayes-best-lower}.
The proof of \cref{eq:avg-optimal-lower-bound}, for the average case in \cref{thm:best-comparator-cgf-average}, can be applied verbatim in the PAC-Bayesian case, as it is only concerned with the structure of $\Delta$ and $\gIDn$.
As these are identical in \cref{thm:best-comparator-cgf-pac-bayes}, the exact same argument can be used, with $B$ in place of $\Bavg$ in \cref{eq:pf-bavg-start} to \cref{eq:pf-bavg-end}.

Next, the result in \cref{eq:cgf-generic-pac-bayes-best-upper-chernoff} follows immediately from \cref{thm:cgf-generic-pac-bayes} by setting $\Delta$ to $\Deltacram$.

We now turn to \cref{eq:cgf-generic-pac-bayes-best-upper-parametric}.
By definition, for any fixed $t$, we have
\begin{equation}\label{eq:re-organize-cgf}
   \E_{\rx \sim \distra{p}}[ e^{ t\mean{\rx} - \Psifunca{p}(t)  } ] = 1 .
\end{equation}
Hence, by applying \cref{thm:cgf-generic-pac-bayes} with $\Delta(q,p)=tq-\Psifunca{p}(t)$, we find that for any fixed $t\in\tspace$, with probability $1-\delta$,
\begin{equation}\label{eq:upper-bound-deriv-fixed-t}
   t\trainlossQ - \Psifunca{\poplossQ}(t) \leq \frac{ \KL( \posterior \Vert \prior ) + \ln\frac{1}\delta }{n} .
\end{equation}
This establishes \cref{eq:cgf-generic-pac-bayes-best-upper-parametric}.

\end{proof}

\begin{leftrule}
\unionbasedbounds*
\end{leftrule}

\begin{proof}
The proofs of these upper bounds are similar to the average case in \cref{thm:best-comparator-cgf-average}, but as we are dealing with a probabilistic result, we need to apply carefully constructed union bounds.

We begin with \cref{eq:cgf-generic-pac-bayes-best-upper-un}.
We will consider the situations where the KL divergence and the training loss are bounded separately, and we begin with the KL case.
Now, note that the supremum over $t$ in \cref{eq:upper-bound-deriv-fixed-t} is achieved for a $t>0$~\citep[Sec.~2.2]{boucheron-13a}.
Hence, we can recast \cref{eq:upper-bound-deriv-fixed-t} as
\begin{equation}\label{eq:upper-bound-deriv-fixed-t-recast}
   \trainlossQ \leq \frac{ \KL( \posterior \Vert \prior ) + \ln\frac{1}\delta }{nt}  + \frac{\Psifunca{\poplossQ}(t)}{t}.
\end{equation}
We now wish to take the infimum over $t$ in the right-hand side in \cref{eq:upper-bound-deriv-fixed-t-recast}, which corresponds to taking the supremum over~$t$ in the left-hand side of \cref{eq:upper-bound-deriv-fixed-t}.

As per our assumption, we have $\KL( \posterior \Vert \prior ) \leq u(n) $.
Let $\rk = \lceil \KL( \posterior \Vert \prior ) \rceil$.
We now follow an approach similar to \citet{rodriguezgalvez-23a}.
Specifically, \cref{eq:upper-bound-deriv-fixed-t-recast} implies 
\begin{align}\label{eq:pf-of-un-disc-bound-rk}
   \trainlossQ &\leq \frac{ \rk + \ln\frac{1}\delta }{nt}  + \frac{\Psifunca{\poplossQ}(t)}{t} .
\end{align}
Now, conditioned on any outcome $\rk=k'$, we can take the infimum over $t$ in the right-hand side of \cref{eq:pf-of-un-disc-bound-rk}:
\begin{align}
   \trainlossQ \leq \inf_t \left\{ \frac{ k'  + \ln\frac{1}\delta }{nt}  + \frac{\Psifunca{\poplossQ}(t)}{t} \right\} .
\end{align}
Note that, given $\rk=k'$, we have $k'\leq \KLterm+1$.
Since the support of $\rk$ is $1,\dots,\lceil u(n) \rceil$, we can apply a union bound over all possible outcomes and perform the substitution $\delta\rightarrow\delta/\lceil u(n) \rceil$ to obtain
\begin{align}\label{eq:disc-union-bound-before-combine-kl}
\Psifunca{\poplossQ}^*(\trainlossQ) = \sup_{t\in\tspace} \left\{ t\trainlossQ - \Psifunca{\poplossQ}(t) \right\}  \leq \frac{ \KL( \posterior \Vert \prior ) + \ln\frac{e\lceil u(n) \rceil}{\delta} }{n}  ,
\end{align}
We now consider the case where $n\trainlossQ \leq u(n) $.
Let $\rs=\lceil n\trainlossQ \rceil$. We then get
\begin{equation}
  \frac{t\rs}{n} - \frac1n + \Psifunca{\poplossQ}(t)  \leq \frac{ \KL( \posterior \Vert \prior ) + \ln\frac{1}\delta }{n} .
\end{equation}
As for the KL divergence, we can optimize the above conditioned on any specific instance of $\rs$, which is supported on $1,\dots,\lceil u(n) \rceil$.
Hence, we apply a union bound over all possible outcomes and perform the substitution $\delta\rightarrow\delta/\lceil u(n) \rceil$ to obtain
\begin{align}\label{eq:disc-union-bound-before-combine-train}
\Psifunca{\poplossQ}^*(\trainlossQ) = \sup_{t\in\tspace} \left\{ t\trainlossQ - \Psifunca{\poplossQ}(t) \right\}  \leq \frac{ \KL( \posterior \Vert \prior ) + \ln\frac{e\lceil u(n)\rceil}{\delta} }{n}  .
\end{align}
By combining \cref{eq:disc-union-bound-before-combine-kl} and \cref{eq:disc-union-bound-before-combine-train} via an additional union bound, performing the substitution $\delta\rightarrow \delta/2$, we obtain \cref{eq:cgf-generic-pac-bayes-best-upper-un}.

We now turn to the upper bound in \cref{eq:cgf-generic-pac-bayes-best-upper}.
Now, we do not assume any bound on the KL divergence or training loss, so we cannot take a union bound over a finite set.
However, we can take the following approach, inspired by \citet{seldin-12a}.
We begin with the case where the minimum in \cref{eq:cgf-generic-pac-bayes-best-upper} is achieved by the KL divergence, and again, we let $\rk=\lceil \KL( \posterior \Vert \prior ) \rceil$.
Then, as before, we have
\begin{align}
 \frac{ \KL( \posterior \Vert \prior ) + \ln\frac{1}\delta }{nt}  + \frac{\Psifunca{\poplossQ}(t)}{t} \leq \frac{ \rk + \ln\frac{1}\delta }{nt}  + \frac{\Psifunca{\poplossQ}(t)}{t} .\label{eq:upper-bound-deriv-fixed-t-recast-k-form}
\end{align}
Conditioned on any outcome $\rk=k'$, we can take the infimum over $t$ in the right-hand side of \cref{eq:upper-bound-deriv-fixed-t-recast-k-form}.
Since $\rk\in\sN_+$, we can now take the following weighted union bound over $\sN$: let $\delta\rightarrow 6\delta/(\pi^2 k'^2)$.
Note that the sum of this over $k'$ is
\begin{equation}
 \sum_{k'\in\sN_+} \frac{6\delta}{\pi^2k'^2} = \delta.   
\end{equation}
We can thus conclude that, with probability $1-\delta$,
\begin{align}
 \trainlossQ \leq \inf_{t\in\tspace} \left\{\frac{ \KL( \posterior \Vert \prior ) + \ln\frac{e \pi^2  (1+\KL( \posterior \Vert \prior ))^2}{6\delta} }{nt}  + \frac{\Psifunca{\poplossQ}(t)}{t} \right\},
\end{align}
and hence,
\begin{align}\label{eq:pf-best-upper-kl}
\Psifunca{\poplossQ}^*(\trainlossQ) = \sup_{t\in\tspace} \left\{ t\trainlossQ - \Psifunca{\poplossQ}(t) \right\}  \leq \frac{ \KL( \posterior \Vert \prior ) + \ln\frac{e \pi^2  (1+\KL( \posterior \Vert \prior ))^2}{6\delta} }{n}   .
\end{align}

Finally, we turn to the upper bound in terms of the training loss in \cref{eq:cgf-generic-pac-bayes-best-upper}.
Let $\rs=\lceil n\trainlossQ \rceil$. Again, we then get
\begin{equation}
  \frac{t\rs}{n} - \frac1n - \Psifunca{\poplossQ}(t)  \leq \frac{ \KL( \posterior \Vert \prior ) + \ln\frac{1}\delta }{n} .
\end{equation}
For any fixed instance of $\rs=m'\in\sN_+$, we can take the supremum over $t$.
So, taking a union bound with $\delta\rightarrow 6\delta/(\pi^2 m'^2)$, we get
\begin{align}\label{eq:pf-best-upper-train}
\Psifunca{\poplossQ}^*(\trainlossQ) \leq \frac{ \KL( \posterior \Vert \prior ) + \ln\frac{\pi^2  (1+\trainlossQ)^2}{6\delta} }{n}   .
\end{align}
The result in \cref{eq:cgf-generic-pac-bayes-best-upper} follows by combining the bounds in \cref{eq:pf-best-upper-kl} and \cref{eq:pf-best-upper-train} via the union bound, performing the substitution $\delta\rightarrow \delta/2$.
\end{proof}

\subsection{Proofs for \cref{sec:applications}}

\begin{leftrule}
\poissonbound*
\end{leftrule}
\begin{proof}
Since the Poisson distributions form a NEF, \cref{propo:kullback-equality} implies that the Cram\'er function is, indeed, equal to the KL divergence in \cref{eq:poisson-kl}.
Hence, \cref{eq:subpoisson-average} follows immediately from \cref{eq:cgf-generic-pac-bayes-best-lower}, while \cref{eq:subpoisson-pac-bayes} follows immediately from \cref{eq:cgf-generic-pac-bayes-best-upper}.
\end{proof}

\begin{leftrule}
\subpoissondiffbound*
\end{leftrule}

\begin{proof}
Let $\rx\sim \poisson(\mu)$.
Then, we have
\begin{equation}
\E[ e^{(1-e^{-t})\mu - t\rx}] = e^{-\mu(1-e^{-t})}e^{(1-e^{-t})\mu} =  1.
\end{equation}
Therefore, it follows from \cref{eq:thm-cgf-generic-average} that
\begin{equation}
(1-e^{-t})\poplossQavg - t\trainlossQavg \leq \frac{\KLtermavg}{n} .
\end{equation}
The result follows by solving for $\poplossQavg$ and taking the infimum over $t$.
\end{proof}

\begin{leftrule}
\subgammabound*
\end{leftrule}
\begin{proof}
    Since the gamma distributions with fixed shape parameter form a NEF, the KL divergence in \cref{eq:kl-gamma} is the Cram\'er function (by \cref{propo:kullback-equality}).
    Hence, \cref{eq:subgamma-average} follows immediately from \cref{eq:cgf-generic-pac-bayes-best-lower}, while \cref{eq:subgamma-pac-bayes} follows from \cref{eq:cgf-generic-pac-bayes-best-upper}.
\end{proof}

\begin{leftrule}
\sublaplacebound*
\end{leftrule}
\begin{proof}
Since the Laplace distributions do \emph{not} form a NEF (unless the mean is fixed), the Cram\'er function cannot be computed on the basis of \cref{propo:kullback-equality}.
Instead, we need to show that \cref{eq:cramer-laplace} is the Cram\'er function via explicit computation.
To this end, note that the CGF for the distribution $\laplace(b,p)$ is, for $\abs{t}\leq 1/b$,
\begin{equation}\label{eq:laplace-cgf}
    \Psi_p(t) = tp - \ln(1-b^2t^2)
\end{equation}
Hence, the Cram\'er function is
\begin{align}
    \Psi^*_p(q) &= \sup_{\abs{t}\leq 1/b} \bigg\{ t(q-p) - \ln(1-b^2t^2) \bigg\} \\
    &= \frac{\sqrt {(q-p)^2+b^2}}{b} - 1 + \ln\!\bigg( \frac{2(b\sqrt{(q-p)^2+b^2} -b^2)}{(q-p)^2} \bigg) ,
\end{align}
where the final step follows by confirming that the maximum is attained at the critical point
\begin{equation}
    t^* = \frac{ \sqrt{ b^2 + (q-p)^2 } }{b(q-p)} - \frac{1}{q-p} .
\end{equation}
With this, the result follows directly from \cref{eq:thm-cgf-generic-average}.
\end{proof}

\begin{leftrule}
\sublaplacediffbound*
\end{leftrule}
\begin{proof}
    Given the form of the CGF for the distribution $\laplace(b,p)$ for $\abs{t}\leq 1/b$, given in \cref{eq:laplace-cgf}, it follows that
    \begin{equation}
        \frac{\gIDtwoarg{\Delta_t}{\bounddistrospacelaplace}}{n} = - \ln(1-b^2t^2).
    \end{equation}
    Hence, it follows from \cref{eq:thm-cgf-generic-average} that
    \begin{equation}
        t(\poplossQavg - \trainlossQavg) \leq \frac{\KLtermavg}{n} - \ln(1-b^2t^2) .
    \end{equation}
    The stated result follows after division by $t$ and taking the infimum.
\end{proof}

\begin{leftrule}
\exponentialcgf*
\end{leftrule}

\begin{proof}
    Given the form of the CGF, we have, for $\rx\sim P_p$,
    \begin{equation}
        \E[e^{t(p-q)}] = e^{ -tp + g(t^2) } e^{tp} = e^{g(t^2)} .
    \end{equation}
    Thus, we conclude that
    \begin{equation}
        \frac{\gIDtwoarg{\Delta_t}{\bounddistrospace}}{n} = g(t^2) .
    \end{equation}
    Therefore, it follows from \cref{eq:thm-cgf-generic-average} that, with $\alpha=\trainlossQavg$ and $\beta = \KLtermavg/n$,
    \begin{equation}\label{eq:diffbased-exp-cgf-bound}
        \poplossQavg  \leq \alpha + \inf_t \bigg\{ \frac{\beta + g(t^2)}{t} \bigg\} .
    \end{equation}
    Note that the bound on $\poplossQavg$ \cref{eq:diffbased-exp-cgf-bound} is the explicit form of $\inf_{t\in \tspace} \Bavg^{\Delta_t}_{n} \big(\alpha,\beta,\gIDarg{\Delta_t}\big)$.
    Now, by reorganizing \cref{eq:diffbased-exp-cgf-bound}, we obtain
    \begin{equation}\label{eq:cramer-exp-cgf-bound}
        \sup_{t\in \tspace} \big\{ t(\alpha - \poplossQavg) - g(t^2)  \big\} \leq  \beta .
    \end{equation}
    Due to the assumed form of the CGF, we see that the left-hand side of \cref{eq:cramer-exp-cgf-bound} is indeed the Cram\'er function:
    \begin{equation}
        \sup_{t\in \tspace} \big\{ t(\alpha - \poplossQavg) - g(t^2)  \big\} = \sup_{t\in \tspace} \big\{ t\alpha - \big( t\poplossQavg + g(t^2)\big)  \big\} = \Deltacram(\alpha,\poplossQavg).
    \end{equation}
    Hence, it implies the bound
    $\poplossQavg \leq \Bavg^{\Deltacram}_{n} \big(\alpha,\beta,1\big)$.
    Thus, the claimed equivalence follows.
\end{proof}

\section{Additional Results}
\label{app:more-results}

In this section, we present some additional results which could not be included in the main text.
In \cref{app:additional-theoretical}, we state and prove some theoretical results: we establish a partial characterization of $\funcspace$ under the sub-$\bounddistrospace$ assumption (\cref{propo:cgf-implies-funcspace}); we show that, under certain conditions, the bound in \cref{thm:cgf-generic-average-disintegrated} always improves upon \cref{thm:cgf-generic-average} (\cref{propo:samplewise-better}); and we provide some additional explicit bounds for various instances of $\bounddistrospace$ (\cref{cor:sub-something-bounds}).
In \cref{app:additional-numerical}, we present additional numerical evaluations of the bounds to support the findings presented in \cref{sec:numeric}.

\subsection{Additional Theoretical Results}\label{app:additional-theoretical}
We begin with a partial characterization of $\funcspace$ under the sub-$\bounddistrospace$ assumption.
\begin{restatable}{proposition}{cgfimplies}\label{propo:cgf-implies-funcspace}
    Assume that $\flint \in \funcspace$ for all~$t\in\sR$.
    Let $g:\sL^2\rightarrow\sR^+$ denote a function that is infinitely differentiable in its first argument.
    Then, $g\in\funcspace$ if it is totally monotone, \ie., for all~$k\in\sN$,
    \begin{equation}\label{eq:absolutely-monotone-functions}
       (-1)^k \frac{\partial^k e^{g(q,p)}}{\partial q^k} \geq 0 .
    \end{equation}
    Furthermore, $g\in\funcspace$ if all of its derivatives are non-negative, \ie,  for all~$k\in\sN$,
    \begin{equation}\label{eq:positive-derivatives}
       \frac{\partial^k e^{g(q,p)}}{\partial q^k} \geq 0 .
    \end{equation}
\end{restatable}
\begin{proof}
Consider a fixed $p$, and let $f(q)\equiv e^{g(q,p)}$.
Any function that satisfies \cref{eq:absolutely-monotone-functions} is said to be totally monotone.
By Bernstein's theorem \citep{bernstein-29a}, there exists a non-negative Borel measure with cumulative distribution function $\varphi$ such that
\begin{equation}
    f(q) = \int_0^\infty e^{-tq}\varphi(t)\dv t. 
\end{equation}
This implies that
\begin{equation}
    \E_{\rvz\sim \datadistro^n}[ f(\trainlossh) ] = \E_{\rvz\sim \datadistro^n} \int_{0}^\infty e^{-t\trainlossh} \varphi(t) \dv t = \int_0^\infty \dv t\varphi(t) \E_{\rvz\sim \datadistro^n}[ e^{-t\trainlossh}  ] ,
\end{equation}
where we used Tonelli's theorem to swap the expectation and integral.
By the assumption that $\flint\in\funcspace$,
\begin{equation}
\int_0^\infty \dv t\varphi(t) \E_{\rvz\sim \datadistro^n}[ e^{-\trainlossh}  ] \leq 
\int_0^\infty \dv t\varphi(t) \E_{\rr'_h\sim P_p^n}[ e^{-t\bar\rr'_h}  ] .
\end{equation}
By swapping the integral and expectation again, 
\begin{equation}
\int_0^\infty \dv t\varphi(t) \E_{\rr'_h\sim P_p^n}[ e^{-t\bar\rr'_h}  ] =  \E_{\rr'_h\sim P_p^n} \int_0^\infty \dv t\varphi(t) e^{-t\bar\rr'_h}   = \E_{\rr'_h\sim P_p^n}[ f(\bar\rr'_h) ].
\end{equation}
This establishes the result under the first condition.
For the second, notice that the function $f^-(q) \equiv g(-q, p)$ is totally monotone.
Hence, we can apply the same arguments as for the first condition.
\end{proof}

Next, we establish that, under certain conditions, the bound in \cref{thm:cgf-generic-average-disintegrated} always improves upon \cref{thm:cgf-generic-average}.

\begin{restatable}{proposition}{samplewisebetter}\label{propo:samplewise-better}
Consider the setting of \cref{thm:cgf-generic-average-disintegrated}.
Then,
    \begin{equation}\label{eq:samplewise-better-propo}
        \frac1n\! \sum_{i=1}^n  \Bavg^{\Delta}_{1}\!\left(\trainlossQavgi , \mutualinfo(\rh;\rz_i) , 1 \right)  \leq  \Bavg^{\Delta}_{n}\!\left(\trainlossQavg , \mutualinfo(\rh;\rvz) , 1 \right) 
    \end{equation}
\end{restatable}
\begin{proof}
To establish the result, we need to show that any value of $\poplossQavg$ that satisfies the bound of the left-hand side of \cref{eq:samplewise-better-propo} also satisfies the bound of the right-hand side.
To this end, assume that for $i\in\{1,\dots,n\}$, we have $p_i\in\sL$ such that
\begin{equation}
    \Delta(\trainlossQavgi, p_i) \leq \mutualinfo(\rh;\rz_i) .
\end{equation}
By averaging over $i$, this implies that
\begin{equation}
    \frac1n \sum_{i=1}^n \Delta(\trainlossQavgi, p_i) \leq \sum_{i=1}^n \frac{\mutualinfo(\rh;\rz_i)}{n} .
\end{equation}
By the convexity of $\Delta$, the left-hand side can be lower-bounded as, with $\bar p = \sum_{i=1}^n p_i/n$,
\begin{equation}
   \Delta(\trainlossQavg, \bar p)  \leq \frac1n \sum_{i=1}^n \Delta(\trainlossQavgi, p_i) .
\end{equation}
Let $\rvz_{<i}=(\rz_1,\dots,\rz_{i-1})$, where $\rvz_{<1}=\emptyset$.
By the chain rule of mutual information and the fact that conditioning on independent random variables increases mutual information,
\begin{equation}
    \sum_{i=1}^n \mutualinfo(\rh;\rz_i) \leq \sum_{i=1}^n 
 \mutualinfo(\rh;\rz_i \vert \rvz_{<i}) = \mutualinfo(\rh;\rvz).
\end{equation}
Thus, it follows that
\begin{equation}
   \Delta(\trainlossQavg, \bar p)  \leq \frac{\mutualinfo(\rh;\rvz)}{n},
\end{equation}
establishing the claim.
\end{proof}

Note that this result is similar to \citet[Prop.~1]{harutyunyan-21a}.

Finally, we provide some additional explicit bounds for various instances of $\bounddistrospace$.
While we only state average bounds explicitly, analogous results hold for the PAC-Bayesian case.

\begin{restatable}{corollary}{subsomething}\label{cor:sub-something-bounds}

Assume that the loss is sub-$\bounddistrospaceinvgauss$, where $\bounddistrospaceinvgauss$ denotes the set of inverse Gaussian distributions with fixed $\lambda$, \ie,
\begin{equation}
    \bounddistrospaceinvgauss = \big\{ \invgauss(\mu,\lambda) : \mu \in \sR \big\} .
\end{equation}
Define
\begin{equation}
    \Deltacramarg{\bounddistrospaceinvgauss}(q,p) = \frac{\lambda(p-q)^2}{2pq^2}.
\end{equation}
Then, we have
\begin{equation}\label{eq:subinvgauss-average}
    \Deltacramarg{\bounddistrospaceinvgauss}(\trainlossQavg,\poplossQavg) \leq \frac{\KLtermavg}{n} .
\end{equation}
Next, assume that the loss is sub-$\bounddistrospacenegbin$, where $\bounddistrospacenegbin$ is the set of negative Binomial distributions with fixed $r$:
\begin{equation}
    \bounddistrospacenegbin = \bigg\{ \negbin\bigg(r,\frac{r}{r+\mu}\bigg) : \mu \in \sR^+ \bigg\} .
\end{equation}
Define
\begin{equation}\label{eq:cramer-negbin}
    \Deltacramarg{\bounddistrospacenegbin}(q,p) =  r\ln\!\left( \frac{p+r}{q+r} \right) + q\ln\!\left( \frac{q(p+r)}{p(q+r)} \right).
\end{equation}
Then, we have
\begin{equation}\label{eq:subnegbin-average}
    \Deltacramarg{\bounddistrospacenegbin}(\trainlossQavg,\poplossQavg) \leq \frac{\KLtermavg}{n} .
\end{equation}
\end{restatable}
\begin{proof}
    Both $\bounddistrospaceinvgauss$ and $\bounddistrospacenegbin$ form NEFs.
    Hence, the results can be established either by computing a Cram\'er function or by computing a KL divergence and applying \cref{propo:kullback-equality}.
    For the inverse Gaussian distribution, the KL divergence is given by~\citep{zhang-07a}
    \begin{equation}
        \KL( \invgauss(q,\lambda)\Vert \invgauss(p,\lambda ) = \frac{\lambda(p-q)^2}{2pq^2} .
    \end{equation}
    Since the inverse Gaussian distributions form a NEF, the result follows by \cref{propo:kullback-equality} and \cref{eq:avg-optimal-lower-bound}.
    Next, for the random variable $\rx\sim\negbin(r,p)$ with $p=r/(r+\mu)$, the Cram\'er function is
    \begin{equation}
        \Psi^*_\mu(q) =  \sup_t \bigg\{ qt - r\log\!\left( \frac{ p }{ 1-(1-p)e^t } \right) \bigg\} .
    \end{equation}
    The optimum can be found by standard techniques, and equals the right-hand side of \cref{eq:cramer-negbin}. Hence, \cref{eq:subnegbin-average} follows by \cref{eq:avg-optimal-lower-bound}.
\end{proof}

\subsection{Additional Numerical Results}\label{app:additional-numerical}

In this section, we present additional numerical results.
The code for reproducing all figures is available as a Jupyter notebook on \codelink, and executes in $5$ minutes on Google Colab CPU.

In \cref{fig:subbernoulli-no-min}, we plot the difference between the binary KL bound and the sub-Gaussian bound, but without the minimum in \cref{eq:bounded-comparison-plot}.
That is, we evaluate
\begin{equation}\label{eq:bounded-discrepancy-no-min}
\Bavg^\kl_n(\alpha,\beta,1) - \big( \alpha + \sqrt{\beta/2n} \big) .
\end{equation}
With this, the biggest discrepancy arises when both the training loss and normalized KL divergence are big, as for the sub-Poissonian case.

\begin{figure}
\centering
\begin{subfigure}{0.33\textwidth}
\centering
\includegraphics[width=\textwidth]{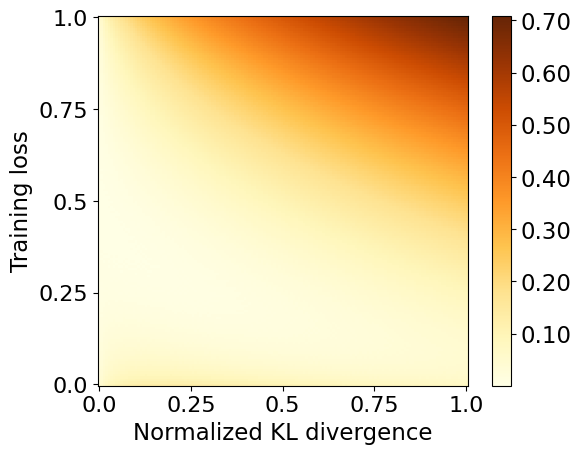}
    \caption{Sub-Bernoulli discrepancy, \cref{eq:bounded-discrepancy-no-min} }%
    \label{fig:subbernoulli-no-min}
\end{subfigure}
\begin{subfigure}{0.33\textwidth}
\centering
\includegraphics[width=\textwidth]{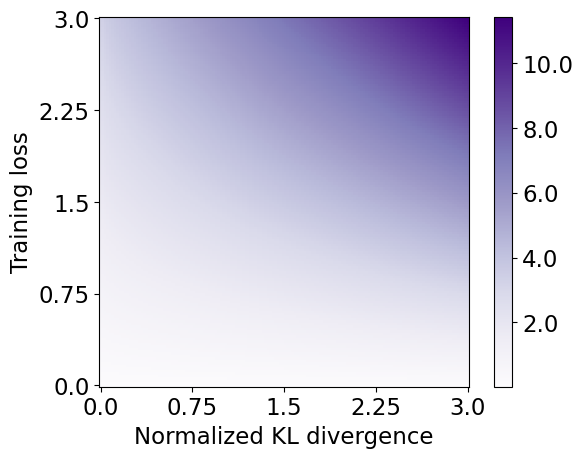}
    \caption{Sub-gamma, \cref{eq:subgamma-average} }%
    \label{fig:subgamma}
\end{subfigure}%
\begin{subfigure}{0.33\textwidth}
\centering
\includegraphics[width=\textwidth]{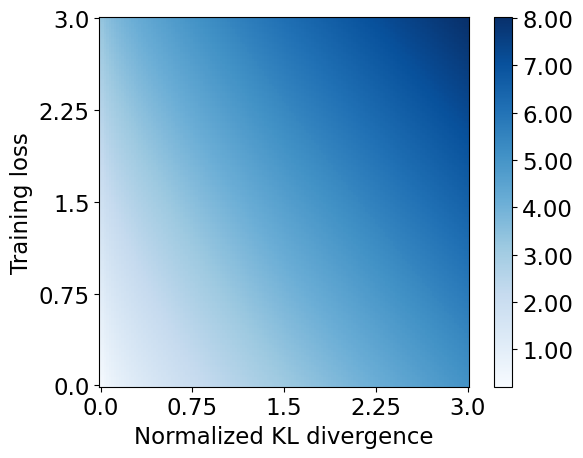}
    \caption{Sub-Laplacian, \cref{eq:sublaplace-average} }%
    \label{fig:sublaplace}
\end{subfigure}

\begin{subfigure}{0.33\textwidth}
\centering
\includegraphics[width=\textwidth]{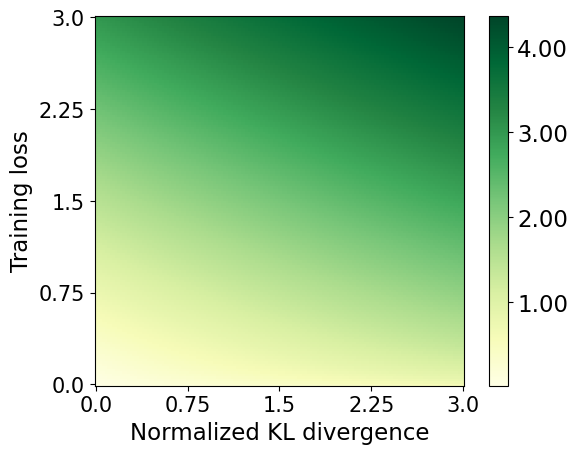}
    \caption{Sub-Poissonian, \cref{eq:subpoisson-average} }%
    \label{fig:subpoisson-our}
\end{subfigure}%
\begin{subfigure}{0.33\textwidth}
\centering
\includegraphics[width=\textwidth]{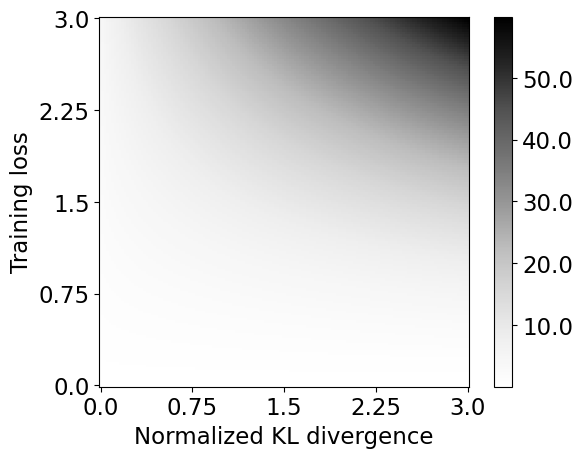}
    \caption{Sub-inverse Gaussian, \cref{eq:subinvgauss-average} }%
    \label{fig:subinvgauss-our}
\end{subfigure}
\begin{subfigure}{0.33\textwidth}
\centering
\includegraphics[width=\textwidth]{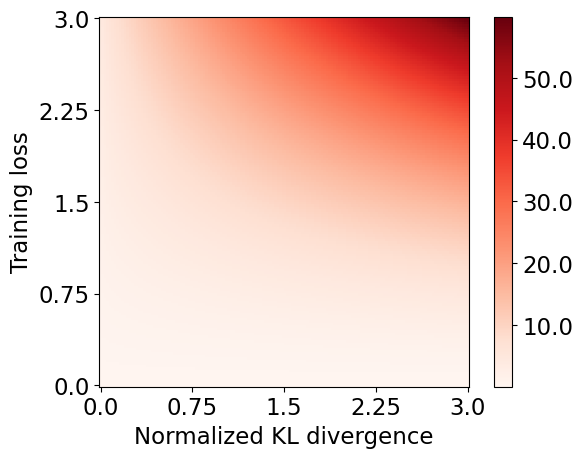}
    \caption{Sub-negative binomial, \cref{eq:subnegbin-average} }%
    \label{fig:subnegbin-our}
\end{subfigure}%
\caption{
In \cref{fig:subbernoulli-no-min}, we plot~\cref{eq:bounded-discrepancy-no-min}.
In \crefrange{fig:subgamma}{fig:subnegbin-our}, we illustrate the numerical values of the Cram\'er bounds.
}
\label{fig:extra_figs}
\end{figure}

Next, in \crefrange{fig:subgamma}{fig:subnegbin-our}, we evaluate our average Cram\'er bounds for sub-gamma, sub-Laplacian, sub-Poissonian, sub-inverse Gaussian, and sub-negative binomial losses.
Specifically, we present the bound based on \cref{eq:subgamma-average} with $k=5$ in \cref{fig:subgamma}; the bound based on \cref{eq:sublaplace-average} with $b=1$ in \cref{fig:sublaplace}; the bound based on \cref{eq:subpoisson-average} in \cref{fig:subpoisson-our}; the bound based on \cref{eq:subinvgauss-average} with $\lambda=1$ in \cref{fig:subinvgauss-our}; and the bound based on \cref{eq:subnegbin-average} with $r=3$ in \cref{fig:subnegbin-our}.
\begin{figure}
\centering
\begin{subfigure}{0.33\textwidth}
\centering
\includegraphics[width=\textwidth]{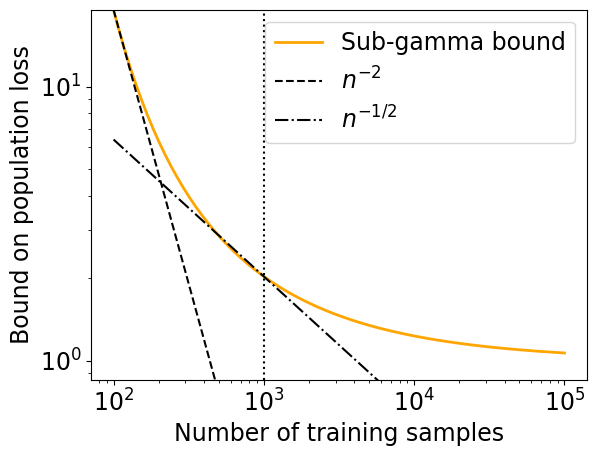}
    \caption{ Sub-gamma, $\alpha=1$ \cref{eq:subpoisson-average} }%
    \label{fig:n-gamma-1}
\end{subfigure}%
\begin{subfigure}{0.33\textwidth}
\centering
\includegraphics[width=\textwidth]{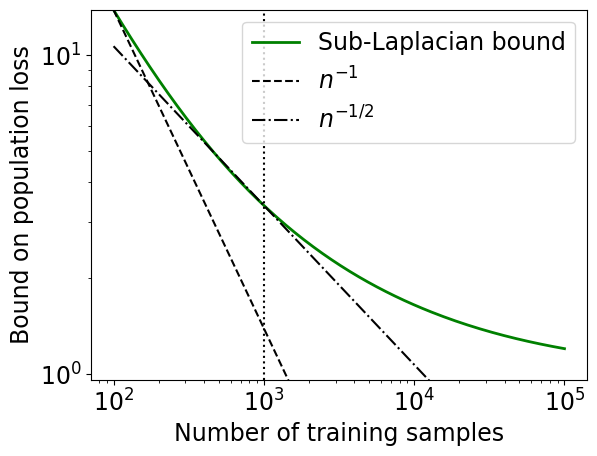}
    \caption{Sub-Laplacian, $\alpha=1$  \cref{eq:subinvgauss-average} }%
    \label{fig:n-laplace-1}
\end{subfigure}
\begin{subfigure}{0.33\textwidth}
\centering
\includegraphics[width=\textwidth]{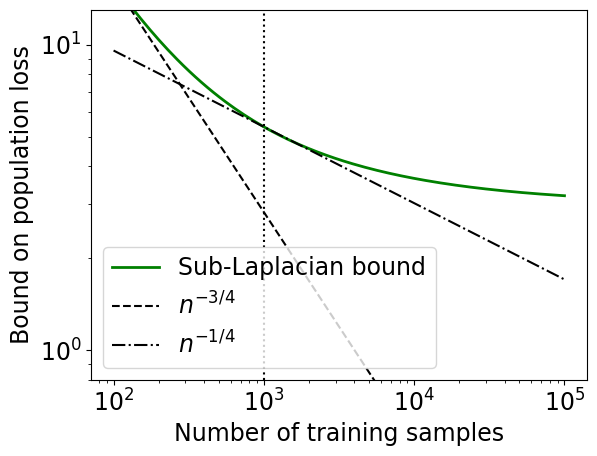}
    \caption{ Sub-Laplacian, $\alpha=3$ \cref{eq:subnegbin-average} }%
    \label{fig:n-laplace-3}
\end{subfigure}%
\caption{
The $n$-dependence of the the Cram\'er bounds for sub-gamma and sub-Laplacian losses.
}
\label{fig:n_figs}
\end{figure}

Finally, in \cref{fig:n_figs}, we numerically study the $n$-dependence of the average bounds in terms of the Cram\'er function for sub-gamma losses, \ie, \cref{eq:subgamma-average}, and for sub-Laplacian losses, \ie, \cref{eq:sublaplace-average}.
Note that, for the purposes of this evaluation, we assume that the training loss $\alpha$ and KL divergence $\beta$ are \emph{fixed}, and only the number of samples $n$ varies.
This is not a realistic assumption in many settings, as both the training loss and KL divergence will typically depend on the sample size---in particular, the KL divergence tends to increase with $n$.
However, this still sheds some light on the behavior of the bounds, and if one knows the dependence of the KL divergence on $n$, this can be incorporated by suitably rescaling.

In \cref{fig:n-gamma-1}, we evaluate \cref{eq:subgamma-average} with $k=5$, training loss $\alpha=1$, and KL divergence $\beta=10^3$.
For the sub-gamma bound, we find a behavior that is consistent across various values of the training loss: initially, the bound decays as $1/n^2$, and when it reaches $n\approx \beta$, it decays as $1/\sqrt n$, after which it further slows.
This demonstrates that, when the number of samples is low ($n\ll \beta$), the bound rapidly improves as more samples are used, while for larger sample sizes ($n\gg \beta$), the improvement is less pronounced.
As a specific example: as $n$ grows from $10^2$ to~$10^3$, the bound decreases by $89\%$, whereas when $n$ grows from $10^4$ to $10^5$, the bound decreases by~$13\%$.

In \cref{fig:n-laplace-1}, we evaluate \cref{eq:sublaplace-average} with $b=1$, training loss $\alpha=1$, and KL divergence $\beta=10^3$, whereas in \cref{fig:n-laplace-3}, we set the training loss to $\alpha=3$.
For the sub-Laplacian bound, the picture is less clear.
For $\alpha=1$, the bound initially decays as $1/n$, while approximating a $1/\sqrt n$ asymptote for $n\approx \beta$ (as for the sub-gamma loss).
However, for $\alpha=3$, the initial decay is closer to $n^{-3/4}$, and for $n\approx \beta$, it is approximately $n^{-1/4}$.

\end{document}